\documentclass[11pt]{article}

\usepackage{wrapfig}
\usepackage{enumitem}
\usepackage{setspace}

\usepackage[utf8]{inputenc}
\usepackage[T1]{fontenc}
\usepackage{microtype}
\usepackage{graphicx}
\usepackage{booktabs}
\usepackage{array}
\usepackage{multirow}
\usepackage{subcaption}
\usepackage{float}
\usepackage{ragged2e}
\usepackage{nicefrac}
\usepackage{xcolor}
\usepackage{adjustbox}
\usepackage{amsmath,amssymb,amsfonts,amsthm,bm}
\usepackage{algorithm}
\usepackage{algpseudocode}
\usepackage[numbers,sort&compress]{natbib}
\usepackage[hyphens]{url}
\usepackage[toc,page]{appendix}
\usepackage[colorlinks=true,linkcolor=blue,citecolor=blue,urlcolor=blue,
            breaklinks=true]{hyperref}

\newcommand{\E}{\mathbb{E}}
\newcommand{\R}{\mathbb{R}}
\newcommand{\V}{\mathcal{V}}
\newcommand{\simplex}{\Delta}
\newcommand{\pc}{p_{\mathrm{c}}}
\newcommand{\DScore}{\mathrm{DScore}}

\newcommand{\softmax}{\mathrm{softmax}}
\newcommand{\argmax}{\operatorname*{arg\,max}}
\newcommand{\argmin}{\operatorname*{arg\,min}}

\newtheorem{proposition}{Proposition}
\newtheorem{lemma}{Lemma}

\newtheorem{remark}{Remark}

\newcommand{\mpSecSetup}{\ref{sec:setup}}
\newcommand{\mpSecMethod}{\ref{sec:method}}
\newcommand{\mpSecToy}{\ref{sec:exp_toy}}
\newcommand{\mpSecByte}{\ref{sec:exp_byte}}
\newcommand{\mpSecScale}{\ref{sec:exp_scale}}
\newcommand{\mpSecRouters}{\ref{sec:exp_routers}}
\newcommand{\mpSecNatural}{\ref{sec:natural}}
\newcommand{\mpEqPool}{\eqref{eq:pool}}
\newcommand{\mpEqSeqPool}{\eqref{eq:seq_pool}}
\newcommand{\mpEqEvidence}{\eqref{eq:evidence}}
\newcommand{\mpEqSeqGradient}{\eqref{eq:seq_gradient}}

\newcommand{\mpEqDecode}{\eqref{eq:decode}}
\newcommand{\mpPropSeqGradient}{\ref{prop:seq_gradient}}
\newcommand{\mpTabByte}{\ref{tab:byte3}}
\newcommand{\mpFigFlagship}{\ref{fig:flagship}}

\title{Evidence-Aligned Local Composition of Discrete Experts\\
       for Sequence Restoration}

\author{Mohammad Panahazari$^{1}$, Usman A. Khan$^{2}$, Shuchin Aeron$^{1}$
\thanks{
$^1$ Department of Electrical and Computer Engineering, Tufts University, Medford, Massachusetts\\
$^2$ Boston College, Chestnut Hill, Massachusetts\\
E-mail: mohammad.panahazari@tufts.edu,
usman.khan@bc.edu, shuchin@ece.tufts.edu}}

\begin{document}
\maketitle

\begin{abstract}
A document modeled as a discrete sequence of tokens can be thought of as being generated from a composition of texts from different domains; a README file, for example, moves between prose, code, and configuration. When such a document is corrupted and only frozen domain experts are available, restoring it requires deciding both what is missing and which expert to trust at each position, at test time and without region labels or a trained router. We introduce evidence-aligned local composition, which infers a soft, position-wise weighting over the experts from the marginal evidence of the corrupted observation under a given corruption model, estimating the evidence from the experts' own denoising losses and smoothing the weights across positions. Because the weighting is soft, it recovers a mixture when the true composition is mixed and concentrates on one expert when that suffices. Across a categorical simulator, byte-level experts, and experts fine-tuned from a $1.3$B discrete flow-matching model, the inferred weights track the true regions at $0.85$ field accuracy on naturally mixed scientific documents, and at $0.98$ on constructed mixtures whose regions are lexically disjoint. Restoration improves over a single global weight when the experts are genuinely distinct and reduces to it when they converge, tracking a measure of expert separation. Code is available at \url{https://github.com/panahazari/evidence-aligned-local-composition}.
\end{abstract}

\section{Introduction}

Real documents mix several kinds of content. A software README interleaves prose, code, and configuration; a scientific article interleaves prose and mathematics; a data-processing script embeds queries in a host language. When such a document is damaged, by imperfect optical character recognition, a lossy channel, or aggressive filtering during data curation, restoring it means recovering the missing symbols. We call this channel a corruption channel, or forward model, and we assume access to it, expressed as the likelihood of the corrupted output conditioned on the true input.

Strong generative priors for each kind of content already exist as frozen, separately trained models, yet no single one serves a whole document well: the best expert changes from position to position, and the regions carry no labels. We study how to restore such a document at test time by composing frozen domain experts, inferring at each position both what symbol is missing and which combination or composition of experts to trust, without position labels or a trained model, i.e. a router for composition. We do not assume a single expert per position. The object we infer is a composition of experts at each position; routing to one expert is the special case in which the composition concentrates on a single vertex of the simplex.

This poses two coupled problems: the missing content must be inferred, but also the composition of the experts that should guide each position. Weighted products of experts, or equivalently weighted sums of their log probabilities, provide a natural way to compose models \citep{Hinton2002,Liu2022,Du2023}. However, most existing approaches assign one weight to each expert for the entire sequence \citep{Mavromatis2024}, which works when expert competence is roughly uniform but not when the preferred composition changes across positions. Local positional weights offer the needed flexibility but are usually supplied by supervision or a trained router, while the test-time mixtures that need no training fix one weight for the whole sequence, chosen from the input's own likelihood \citep{Mavromatis2024,wPoE2025}; Table~\ref{tab:related} places these lines along the three axes that define our setting.  Our question is different: Can a soft field of expert weights be inferred at test time from the corrupted observation and the given corruption channel, without region labels or router training? 

We approach this question through marginal evidence: the position-dependent weights parameterize a composed prior, and the likelihood of the observation determines how they should change \citep{MacKay1992}, connecting to evidence optimization and model selection for continuous-diffusion priors \citep{WangBouman2026,DiME2026} but making the weights local and discrete. The resulting update compares each expert's fit to the observation with its fit under the current composition, a posterior-minus-prior difference we make precise in Section~\ref{sec:method}. Because the exact update is not directly computable, we approximate it, which motivates the name \emph{evidence-aligned local composition}.

\begin{table}[t]
\centering
\small
\caption{Methods along our three axes and their data domain. \emph{Local}: the weights vary by position within one input, rather than one weighting being applied to the whole input. \emph{Test-time}: the weights are set at inference from the input itself, with no router or weight predictor trained beforehand. \emph{Corruption-aware}: inferred from the corrupted observation under a given channel. wPoE optimizes a single scalar at test time from one data point, so it is test-time but not local.}
\label{tab:related}
\begin{adjustbox}{max width=\linewidth}
\begin{tabular}{llcccc}
\toprule
Method & Experts & Domain & Local & Test-time & Corruption-aware \\
\midrule
Composable diffusion \citep{Liu2022,Du2023} & continuous diffusion & image & $\times$ & \checkmark & $\times$ \\
DEMix \citep{DEMix2022} & autoregressive LMs & text & $\times$ & \checkmark & $\times$ \\
PackLLM \citep{Mavromatis2024} & autoregressive LMs & text & $\times$ & \checkmark & $\times$ \\
wPoE \citep{wPoE2025} & autoregressive LMs & text & $\times$ & \checkmark & $\times$ \\
FactorDiff \citep{FactorDiff2026} & discrete diffusion & image & \checkmark & \checkmark & $\times$ \\
Evidence-optimized priors \citep{WangBouman2026} & continuous diffusion & image & $\times$ & \checkmark & \checkmark \\
\midrule
Ours & discrete flow matching & text & \checkmark & \checkmark & \checkmark \\
\bottomrule
\end{tabular}
\end{adjustbox}
\end{table}

\paragraph{Main Contributions.}
\begin{itemize}\itemsep2pt
\item \textbf{Local composition from corrupted-observation evidence.} We define a normalized sequence energy model with a simplex-valued expert weight at every position and infer it at test time by an exponentiated-gradient update on the marginal evidence, smoothed across positions. The evidence gradient compares posterior and prior expectations of the same local expert energy.
\item \textbf{Denoising losses as local energies.} We score candidate sequences with the experts' denoising losses. For mask-source experts, the schedule-weighted population loss upper-bounds sequence negative log likelihood; an unrestricted Bayes-optimal denoiser recovers the associated any-order conditional decomposition on supported contexts. The deployed unweighted score remains a surrogate.
\item \textbf{Field recovery without supervision.} Through extensive experiments, we show that, without region labels, the inferred field closely matches the true regions and improves restoration over an independently optimized global weight on constructed mixtures; on natural documents the gain tracks how distinct the experts are, improving restoration over the global family on scientific text mixing prose and mathematics and tying when the experts converge.
\end{itemize}

\section{Related Work}
\label{sec:related}

\paragraph{Composing and fusing experts.}
Composing models by multiplying densities, or summing log scores, underlies products of experts \citep{Hinton2002} and compositional or superposed diffusion \citep{Liu2022,Du2023,Garipov2023,Skreta2025}; frozen language models are similarly fused or merged without training, as in PackLLM's perplexity-optimized global mixture \citep{Mavromatis2024}, DEMix's parameter-free domain posterior \citep{DEMix2022}, wPoE's single test-time scalar for compression \citep{wPoE2025}, and output- or parameter-level ensembling and merging \citep{Wan2024,Jiang2023,Sukhbaatar2024}. All of these assign one weighting to the whole input. Position-dependent weights are otherwise supplied by supervision or, in concurrent work, produced by factorized composition of discrete diffusion experts from confidence margins \citep{FactorDiff2026}. In contrast, our weight varies with position and is inferred at test time from the corrupted observation and a given channel, with no router training or region labels.

\paragraph{Discrete diffusion priors and inverse problems.}
Our experts are discrete diffusion and flow-matching models \citep{Austin2021,Campbell2022,Lou2024,Sahoo2024,Shi2024,Gat2024,Lipman2023} that now scale to billions of parameters \citep{Nie2025Scaling,Gong2025,Nie2025LLaDA,Dream2025}; for mask-source absorbing diffusion, the schedule-weighted denoising objective upper-bounds sequence negative log likelihood and a Bayes-optimal denoiser recovers clean-data conditionals and their expected any-order autoregressive decomposition \citep{Ou2024,Uria2014,Hoogeboom2022}. Restoring an observation under such a prior is the subject of diffusion inverse problems, in the continuous setting \citep{Chung2023,Kawar2022,WangBouman2026} and, more recently, with guidance tailored to discrete state spaces \citep{Nisonoff2025,Schiff2025}, while marginal-evidence estimation also drives model selection \citep{DiME2026,MacKay1992}. Our corruption channel plays the role of the forward operator, but the quantity we infer is a local composition field over frozen experts rather than a single guided reconstruction.

\section{Problem Formulation}
\label{sec:setup}

\paragraph{Problem setting.}
Let $\V$ be a finite vocabulary and $x=(x_1,\dots,x_L)\in\V^L$ a clean sequence, where $\ell \in [1:L]$ is the position index. The goal is to recover x from y. In order to do this, We are given $m$ generative experts $q_1,\dots,q_m$, each defining a distribution over $\V^L$. These experts are contextual: the per-position conditional $q_i(x_\ell\mid x_{\setminus\ell})$ over $\V$ depends on the other positions $x_{\setminus\ell}$, as in the discrete flow-matching models we use. A position-factorized expert, whose conditional ignores $x_{\setminus\ell}$, is the special case we use for the local pool in Eq.~\eqref{eq:pool}. The experts remain frozen during composition, and their training data are not available. We observe only $y$, produced from $x$ by a given corruption channel $\pc(y\mid x)$. In the \emph{mask channel}, each token is independently replaced with probability $r$ by a symbol $\mathtt{[MASK]}\notin\V$; in the \emph{replacement channel}, it is replaced by a uniform draw from $\V$. Both channels factorize across positions, so $\pc(y\mid x)=\prod_\ell \pc(y_\ell\mid x_\ell)$. We denote the local likelihood of candidate token $v$ by $c_\ell(v)=\pc(y_\ell\mid x_\ell=v)$.

\paragraph{Local composition.}
We attach a simplex-valued weight to every position. Specifically, the \emph{composition field} is $\lambda_{1:L}=(\lambda_1,\dots,\lambda_L)$ with $\lambda_\ell\in\simplex_m=\{\lambda\in\R_+^m:\sum_i\lambda_i=1\}$. If expert $i$ assigns the categorical distribution $\pi_{i,\ell}$ to the token at position $\ell$, their local log-linear pool is
\begin{equation}\pi_{\lambda_\ell,\ell}(v)=\frac{\exp\big(\textstyle\sum_i \lambda_{\ell,i}\log\pi_{i,\ell}(v)\big)}{\sum_{u\in\V}\exp\big(\sum_i \lambda_{\ell,i}\log\pi_{i,\ell}(u)\big)}. \label{eq:pool}\end{equation}
A vertex $\lambda_\ell=e_i$ (the $i$-th standard basis vector of $\R^m$) routes the position to the single $i$-th expert, whereas an interior point composes several experts. We say a sequence has a \emph{region structure} when its positions split into contiguous blocks, each generated by a single expert; the field is then near a vertex inside a block and mixed near a boundary. Region labels are never given to the method. Where they exist, we use them only to evaluate the inferred field (Section~\ref{sec:experiments}). The local pool of Eq.~\eqref{eq:pool} composes the experts one position at a time and is well defined only when the experts are position-factorized. Our experts are contextual, so a token's conditional depends on the rest of the sequence, and the product of their local conditionals need not be a valid joint distribution. We therefore, compose the experts at the level of the whole sequence, and recover the local pool as the position-factorized special case.

\paragraph{Sequence-level composition.}
The local pool in \eqref{eq:pool} is sufficient for position-factorized experts. The contextual experts here are the same $q_1,\dots,q_m$ defined above; there is no separate set of models. For these experts, however, multiplying their position-wise conditionals need not produce a normalized joint distribution. We therefore define the sequence model directly. Let $\phi_{i,\ell}(x)\in\R$ denote the energy that expert $i$ assigns to position $\ell$ within sequence $x$; Section~\ref{sec:method} later constructs its practical surrogate. The resulting normalized log-linear prior is
\begin{equation}
\begin{aligned}
Q_\lambda(x)&=\tfrac{1}{Z(\lambda)}\exp\Big(\textstyle\sum_{\ell=1}^{L}\sum_{i=1}^{m}\lambda_{\ell,i}\,\phi_{i,\ell}(x)\Big),\\
Z(\lambda)&=\textstyle\sum_{x'\in\V^L}\exp\Big(\sum_{\ell,i}\lambda_{\ell,i}\,\phi_{i,\ell}(x')\Big).
\end{aligned}
\label{eq:seq_pool}
\end{equation}
Its marginal likelihood for the observation is $p_\lambda(y)=\sum_x\pc(y\mid x)Q_\lambda(x)$. When all experts are position-factorized and $\phi_{i,\ell}(x)=\log\pi_{i,\ell}(x_\ell)$, $Q_\lambda$ factorizes and reduces to the local construction in \eqref{eq:pool}. For contextual experts, the evidence under \eqref{eq:seq_pool} remains our target, although the practical sampler is not known to draw from $Q_\lambda$. Section~\ref{sec:method} makes this approximation explicit.

\paragraph{Objective.}
Given the composed prior \eqref{eq:seq_pool}, the field is inferred from the corrupted observation alone. Define the log marginal likelihood of the observation, with the unknown clean sequence integrated out, and the total variation of the field along the sequence:
\begin{equation}
\begin{gathered}
L(\lambda;y)=\log \textstyle\sum_{x\in\V^L} \pc(y\mid x)\,Q_\lambda(x),\quad
\Omega(\lambda)=\textstyle\sum_{\ell=2}^{L}\|\lambda_\ell-\lambda_{\ell-1}\|_1,\\
\hat\lambda_{1:L}\in\argmax_{\lambda_{1:L}\in\simplex_m^L} F(\lambda;y),\quad
F(\lambda;y)=L(\lambda;y)-\tau\,\Omega(\lambda),\ \ \tau\ge0.
\end{gathered}
\label{eq:seq_evidence}
\end{equation}
The regularizer admits a prior reading. A Laplace density on the increments $\lambda_\ell-\lambda_{\ell-1}$ with scale $\tau^{-1}$ has $\log p(\lambda)=-\tau\Omega(\lambda)$ up to an additive constant, so $F$ is the log posterior $\log p(y\mid\lambda)+\log p(\lambda)$. Penalizing the increments in $\ell_1$ makes the maximizer piecewise constant, matching regions that persist and then change at a boundary. At $\tau=0$ the objective reduces to the evidence.

When $N$ corrupted observations share one field, as in the simulator of Section~\ref{sec:exp_toy}, we average $L(\lambda;y^{(n)})$ over $n$; every real-data experiment has one observation per document.

For position-factorized experts and a memoryless channel (one that acts on each position independently, $\pc(y\mid x)=\prod_\ell\pc(y_\ell\mid x_\ell)$, as with the mask and replacement channels above), the evidence term separates across positions; this is proved as Proposition~\ref{prop:gradient} in Appendix~\ref{app:factorized}. Only $\Omega$ couples adjacent positions, so at $\tau=0$ the objective decouples and each $\lambda_\ell$ is optimized independently. This is the case used in the controlled simulator. Writing $c_{n,\ell}(v)=\pc(y^{(n)}_\ell\mid x_\ell=v)$, the contribution of position $\ell$ to the $N$-observation average becomes
\begin{equation}
\begin{gathered}
L_\ell(\lambda_\ell)=\tfrac1N\textstyle\sum_{n=1}^N\log \sum_{v\in\V} c_{n,\ell}(v)\,\pi_{\lambda_\ell,\ell}(v),\\
\hat\lambda_\ell\in\argmax_{\lambda_\ell\in\simplex_m} L_\ell(\lambda_\ell).
\end{gathered}
\label{eq:evidence}
\end{equation}
We use this tractable case for controlled validation. For contextual experts, the sequence-level objective \eqref{eq:seq_evidence} remains the target.

\paragraph{Expert requirements.}
Each expert $q_i$ must expose two operations for our method; we do not introduce new models here. A discrete flow-matching model, equivalently a mask-source discrete diffusion model, generates a sequence from a fully masked state by iteratively revealing tokens along a time schedule: at each step a denoiser predicts the clean token at every masked position from the revealed context, trained by a schedule-weighted denoising cross-entropy whose population value upper-bounds sequence negative log likelihood \citep{Austin2021,Campbell2022,Lou2024,Gat2024,Sahoo2024,Shi2024}. Unlike autoregressive models the denoiser is non-causal, conditioning on both sides of a masked position, which makes these models natural priors for restoration; we use them as our experts. Given a partially masked context $z$ and time $t$, each expert returns logits $h_i(z,t)_{\ell,v}$ for the clean token at each position, with softmax $p_{\theta_i}(x_\ell=v\mid z,t)$. It must also evaluate its denoising loss on a candidate clean sequence. Beyond these operations, the method is agnostic to architecture and model size.

\section{Method}
\label{sec:method}

Proofs of the statements in this section are collected in Appendix~\ref{app:proofs}, apart from the short argument for Lemma~\ref{lem:minorant}, which we keep in place because it is what makes the update a minorize-maximize step.

\subsection{The evidence gradient}

The update direction is the gradient with respect to $\lambda$ of the sequence-level evidence \eqref{eq:seq_evidence}. The simplex constraint $\lambda_\ell\in\simplex_m$ is not imposed on the gradient itself; it is maintained by the multiplicative exponentiated-gradient update in Eq.~\eqref{eq:mw}, which renormalizes each $\lambda_\ell$ back onto the simplex after every step. The deployed sampler is not known to draw from $Q_\lambda$, so the result below describes the exact target rather than the practical estimator.

\begin{proposition}[Sequence-level evidence gradient]
\label{prop:seq_gradient}
Let $Q_\lambda$ be the normalized composed prior \eqref{eq:seq_pool} with finite energies $\phi_{i,\ell}$, let $\pc(y\mid x)$ be a channel with $\sum_x \pc(y\mid x)Q_\lambda(x)>0$, and let $Q_\lambda(x\mid y)\propto \pc(y\mid x)Q_\lambda(x)$ be the induced posterior. Then $L(\lambda;y)=\log p_\lambda(y)$ is differentiable in $\lambda$ on the relative interior of $\simplex_m^L$ and, for every position $\ell$ and expert $i$,
\begin{equation}
\begin{aligned}
\frac{\partial L(\lambda;y)}{\partial\lambda_{\ell,i}}={}&\E_{x\sim Q_\lambda(\cdot\mid y)}\big[\phi_{i,\ell}(x)\big]\\
&-\E_{x\sim Q_\lambda}\big[\phi_{i,\ell}(x)\big].
\end{aligned}
\label{eq:seq_gradient}
\end{equation}
\end{proposition}

Appendix~\ref{app:gradient} gives the proof. The gradient has a simple interpretation: it compares the same local expert energy under sequences consistent with the observation and under sequences drawn from the composed prior. A coordinate increases relative to the others when its expert has a larger posterior-minus-prior energy gap. We call this difference the \emph{evidence correction}; Section~\ref{sec:exp_routers} compares it with local heuristics that omit the prior term. Expert-specific additive offsets cancel in the subtraction, but relative scales do not, so the correction alone does not calibrate the experts. At a masked position the local likelihood is flat, so in a factorized model the evidence gradient vanishes there (Proposition~\ref{prop:mask_degeneracy}, proved in Appendix~\ref{app:factorized}); evidence then reaches the position only through contextual dependence in the sequence model or through smoothing from neighbors, the two mechanisms our single-observation real-data experiments rely on.

\subsection{Denoising losses as local expert energies}

The exact gradient requires local expert energies. Recall from Eq.~\eqref{eq:seq_pool} that the local expert energy $\phi_{i,\ell}(x)\in\R$ is the contribution of position $\ell$ to the log score expert $i$ assigns to the whole sequence $x$, so that expert $i$'s total log score is $\sum_\ell\phi_{i,\ell}(x)$ up to the normalizer; a higher $\phi_{i,\ell}(x)$ means expert $i$ finds the token at position $\ell$ more compatible with its context. Flow-matching models do not expose normalized sequence densities directly, but their denoising losses provide a tractable surrogate. Let $p_t(\cdot\mid x)$ denote the mask-source path that reveals each position of candidate sequence $x$ independently with probability $\kappa(t)$ at time $t$ and masks it otherwise, where $\kappa$ is an increasing schedule with $\kappa(0)=0$ and $\kappa(1)=1$. The schedule-weighted population denoising loss is a variational upper bound on sequence negative log likelihood \citep{Austin2021,Campbell2022,Sahoo2024,Shi2024}. For an unrestricted Bayes-optimal denoiser, it recovers clean-data conditionals on contexts with positive $q_i$-probability. Their order-averaged negative log scores are $\mathcal E^\star_{i,\ell}(x)=\E_\sigma[-\log q_i(x_\ell\mid x_{P_\sigma(\ell)})]$, where $\sigma$ is a uniformly random order and $P_\sigma(\ell)$ contains the positions preceding $\ell$ \citep{Uria2014,Hoogeboom2022,Ou2024}. By the chain rule these scores sum to $-\log q_i(x)$. For a factorized expert, $\mathcal E^\star_{i,\ell}$ is the token negative log density, while $\phi_{i,\ell}(x)=-\mathcal E^\star_{i,\ell}(x)$ is the token log density. In practice, we approximate this energy using $K$ sampled times:
\begin{equation}
\begin{gathered}
\widehat{\DScore}_{i,\ell}(x)=\tfrac1K\textstyle\sum_{k=1}^K -\log \softmax_{v=x_\ell} h_i(x_{t_k},t_k)_{\ell,v},\\
\hat s_{i,\ell}(x)=-\alpha_i\,\widehat{\DScore}_{i,\ell}(x),
\end{gathered}
\label{eq:dscore}
\end{equation}
where $t_k\sim\mathrm{U}[0,1]$ and $x_{t_k}\sim p_{t_k}(\cdot\mid x)$. Because $\widehat{\DScore}_{i,\ell}$ omits the schedule weight, we call it a denoising score rather than a Negative ELBO (NELBO). We use $\hat s_{i,\ell}$ as a surrogate for $\phi_{i,\ell}$; in the factorized case, it plays the role of the token log density $\log\pi_{i,\ell}$. The negative sign maps a lower denoising loss to greater compatibility. Additive expert-specific offsets cancel from the evidence correction, whereas the scales $\alpha_i$ do not. We set $\alpha_i=1$ throughout. The specialization check measures whether each expert favors its intended domain, but it does not calibrate these relative scales.

To estimate the two expectations, we compose the logits as $h_\lambda(z,t)_{\ell,v}=\sum_i\lambda_{\ell,i}h_i(z,t)_{\ell,v}$. Given $h_\lambda$ and observation $y$, the sampler $\mathcal S(h_\lambda,y)$ returns an unclamped particle set $\mathcal X_{\mathrm{pr}}=\{x_{\mathrm{pr}}^{(b)}\}_{b=1}^P$ and an observation-clamped set $\mathcal X_{\mathrm{po}}=\{x_{\mathrm{po}}^{(b)}\}_{b=1}^P$. We use their empirical averages for $\E_{\mathrm{prior}}$ and $\E_{\mathrm{post}}$, respectively. Averaging the surrogate over these two particle sets gives the implemented gradient estimate for $\partial L(\lambda;y)/\partial\lambda_{\ell,i}$,
\begin{equation}
\begin{aligned}
g_{\ell,i}={}&-\alpha_i\,\E_{\mathrm{post}}\big[\widehat{\DScore}_{i,\ell}(x)\big]\\
&+\alpha_i\,\E_{\mathrm{prior}}\big[\widehat{\DScore}_{i,\ell}(x)\big].
\end{aligned}
\label{eq:est_gradient}
\end{equation}
\paragraph{Relation to the exact objective.}
Equation~\eqref{eq:seq_gradient} is exact for the normalized model \eqref{eq:seq_pool}, whereas \eqref{eq:est_gradient} replaces its energies with finite-sample denoising scores, omits the schedule weight, scores revealed positions, and uses a composed-logit sampler not known to target $Q_\lambda$. Thus the deployed score is not asserted to be a likelihood bound. We therefore call the method evidence-aligned rather than evidence-maximizing and compare it with the exact objective in Section~\ref{sec:exp_toy}.

\subsection{Optimization, smoothing and decoding}

The two terms of \eqref{eq:seq_evidence} call for different treatment: $L$ is smooth in $\lambda$ but not concave, while $\Omega$ is convex but nonsmooth. We therefore take a gradient step on $L$ and a proximal step on $\Omega$. Throughout, $g^{(s)}$ denotes the estimate \eqref{eq:est_gradient} of $\nabla_\lambda L$ at the current iterate $\lambda^{(s)}$; the derivation is stated for the exact gradient, and $g^{(s)}$ substitutes for it in the deployed algorithm.

\paragraph{The evidence is a difference of convex functions.}
Write $c(x)=\pc(y\mid x)$ and recall $\Phi_\lambda(x)=\sum_{\ell,i}\lambda_{\ell,i}\phi_{i,\ell}(x)$, which is affine in $\lambda$. Splitting the evidence into its two normalizers,
\begin{equation}
L(\lambda;y)=\log\textstyle\sum_x c(x)e^{\Phi_\lambda(x)}-\log\textstyle\sum_x e^{\Phi_\lambda(x)},
\label{eq:dc}
\end{equation}
exhibits it as a difference of two log-sum-exp functions of affine arguments, both convex. Since $\V^L$ is finite and every summand is positive, $L$ is smooth; it is not concave, and \eqref{eq:seq_evidence} is not a concave program.

\paragraph{A concave surrogate, tight at the current iterate.}
\begin{lemma}[Minorant]
\label{lem:minorant}
Fix $\lambda^{(s)}$ in the relative interior of $\simplex_m^L$ and set {$M(\lambda\mid\lambda^{(s)})=\E_{x\sim Q_{\lambda^{(s)}}(\cdot\mid y)}\big[\Phi_\lambda(x)\big]-\log Z(\lambda)$}. Then \emph{(i)} $L(\lambda;y)\ge M(\lambda\mid\lambda^{(s)})+\kappa$ for a constant $\kappa$ independent of $\lambda$, with equality at $\lambda=\lambda^{(s)}$; \emph{(ii)} $M(\cdot\mid\lambda^{(s)})$ is concave; and \emph{(iii)} its gradient at $\lambda^{(s)}$ is the evidence gradient \eqref{eq:seq_gradient}.
\end{lemma}
\begin{proof}
For \emph{(i)}, apply Jensen to the first term of \eqref{eq:dc} using the weights $Q_{\lambda^{(s)}}(x\mid y)\propto c(x)e^{\Phi_{\lambda^{(s)}}(x)}$, which are positive and sum to one; $\kappa$ collects their entropy and $\E[\log c]$. Jensen is tight when the weights are proportional to the summand, that is at $\lambda=\lambda^{(s)}$. For \emph{(ii)}, $\E[\Phi_\lambda(x)]$ is affine in $\lambda$ and $\log Z$ is convex, so the difference is concave. For \emph{(iii)}, the first term is affine in $\lambda$ with its expectation taken under the \emph{fixed} law $Q_{\lambda^{(s)}}(\cdot\mid y)$, so it differentiates to $\E_{x\sim Q_{\lambda^{(s)}}(\cdot\mid y)}[\phi_{i,\ell}(x)]$, while $\partial\log Z(\lambda)/\partial\lambda_{\ell,i}=\E_{x\sim Q_\lambda}[\phi_{i,\ell}(x)]$ by the exponential-family identity used in Proposition~\ref{prop:seq_gradient}. Setting $\lambda=\lambda^{(s)}$ turns the second expectation into one under $Q_{\lambda^{(s)}}$ and recovers \eqref{eq:seq_gradient}.
\end{proof}

Proposition~\ref{prop:seq_gradient} is thus the gradient of a concave surrogate that touches $L$ at the current iterate. Since $-\tau\Omega$ is concave, $M-\tau\Omega$ minorizes $F$, and maximizing it exactly would give monotone ascent, $F(\lambda^{(s+1)})\ge M(\lambda^{(s+1)}\mid\lambda^{(s)})-\tau\Omega(\lambda^{(s+1)})+\kappa\ge F(\lambda^{(s)})$. That maximization is intractable, as $M$ contains $\log Z$. We therefore take one composite mirror ascent step on its linearization, with the entropic mirror map on each row:
\begin{equation}
\lambda^{(s+1)}=\argmax_{\lambda\in\simplex_m^L}\Big\{\langle g^{(s)},\lambda\rangle-\tau\Omega(\lambda)-\tfrac1\eta\textstyle\sum_\ell D_{\mathrm{KL}}\big(\lambda_\ell\,\|\,\lambda^{(s)}_\ell\big)\Big\}.
\label{eq:mirror}
\end{equation}
The entropic map is the natural geometry on a product of simplices: it keeps every iterate in the relative interior without an explicit projection, and its Bregman divergence is the KL divergence.

\paragraph{Splitting \eqref{eq:mirror} into two exactly solved steps.}
Equation \eqref{eq:mirror} has no closed form, because $\Omega$ couples adjacent rows. Dropping $\Omega$ leaves a problem that does: each row of the $\tau=0$ case of \eqref{eq:mirror} is solved exactly by
\begin{equation}
\tilde\lambda^{(s+1)}_{\ell,i}=\frac{\lambda^{(s)}_{\ell,i}\exp\!\big(\eta\, g^{(s)}_{\ell,i}\big)}{\sum_{j}\lambda^{(s)}_{\ell,j}\exp\!\big(\eta\, g^{(s)}_{\ell,j}\big)}.
\label{eq:mw}
\end{equation}
The normalization in \eqref{eq:mw} is the Bregman projection onto the simplex: dividing by the row sum is the exact solution of $\argmin_{\lambda_\ell\in\simplex_m}D_{\mathrm{KL}}\big(\lambda_\ell\,\|\,\lambda^{(s)}_\ell e^{\eta g^{(s)}_\ell}\big)$. Equation \eqref{eq:mw} is therefore a single mirror ascent step, not a gradient step followed by a separate renormalization.

The penalty is then applied as a proximal step from that point,
\begin{equation}
\lambda^{(s+1)}
=\argmin_{\lambda_{1:L}\in\simplex_m^L}\Big\{\tfrac12\sum_{\ell}\big\|\lambda_\ell-\tilde\lambda^{(s+1)}_\ell\big\|_2^2+\tau \Omega(\lambda)\Big\},
\label{eq:prox}
\end{equation}
the step $\mathrm{Smooth}_\tau$ of Algorithm~\ref{alg:main}, which returns a piecewise-constant field. The simplex constraint sits inside the argmin, so \eqref{eq:prox} is the proximal map of $\tau\Omega$ together with the constraint indicator, not of $\tau\Omega$ alone.

The order of the two steps is determined by scale. The quadratic term in \eqref{eq:prox} is not scale invariant, and the unnormalized vector $\lambda^{(s)}_\ell e^{\eta g^{(s)}_\ell}$ has total mass varying with $g^{(s)}$. Applying \eqref{eq:prox} before \eqref{eq:mw} would therefore let the effective strength of $\tau$ drift with the gradient magnitude across iterations. Normalizing first fixes the scale, and it makes $\Omega$ measure variation in composition rather than in magnitude.

For two experts, $\lambda_\ell=(a_\ell,1-a_\ell)$ reduces \eqref{eq:prox} to $\min_a\tfrac12\sum_\ell(a_\ell-\tilde a_\ell)^2+\tau\sum_{\ell\ge2}|a_\ell-a_{\ell-1}|$ up to an overall factor of two, the scalar fused lasso \citep{Tibshirani2005}, which Condat's algorithm solves exactly in linear time \citep{Condat2013}. Its solution cannot leave the range of its input, so it stays in $[0,1]$, the constraint is slack, and \eqref{eq:prox} is solved exactly; this covers every two-expert study in the paper. For three or more experts we apply the same map per column and renormalize, which does not maximize \eqref{eq:prox}, since the proximal map of a sum is not the composition of the proximal maps.

\paragraph{Scope of the derivation.}
Each step solves a stated problem exactly: \eqref{eq:mw} is the $\tau=0$ case of \eqref{eq:mirror}, and \eqref{eq:prox} is solved exactly for two experts. Their composition, however, is not \eqref{eq:mirror}: the forward step is taken in the KL geometry and the backward step in the Euclidean one, and a composition of proximal maps is not in general the proximal map of the sum. With the nonconcavity of \eqref{eq:dc}, this places the procedure as a splitting with a structural prior, for which we claim no convergence. The experiments on naturally mixed documents instead use a low-pass filter, convolving each column with a uniform kernel and blending with the unsmoothed field; it attenuates variation uniformly rather than preserving edges, and its $\tau$ is a blend weight, not comparable with \eqref{eq:prox}. The constructed studies use the proximal step at every scale, the naturally mixed document studies use the filter, and the simulator uses neither; Appendix~\ref{app:hparams} gives each experiment's form and $\tau$.

Once the field has been inferred, we restore each masked token by selecting the per-position mode of the composed logits,
\begin{equation}
\hat x_\ell=\argmax_{v\in\V}\textstyle\sum_i \hat\lambda_{\ell,i}\,h_i(z,t_{\mathrm{dec}})_{\ell,v},\quad \ell\in M,
\label{eq:decode}
\end{equation}
where $M$ is the masked set, $z$ contains the visible tokens, and all visible positions are copied unchanged. The decoder fills masked positions independently and therefore returns per-position modes rather than the joint posterior mode of $Q_\lambda$. At byte scale, decoding is deterministic given the field. The pretrained decoder also samples a uniform-source initialization, fixed once per document and shared across methods to pair decoding randomness; the iterative-decoder ablation preserves the ordering. Algorithm~\ref{alg:main} summarizes the procedure using the score $\hat s$ of \eqref{eq:dscore}.

\begin{algorithm}[t]
\caption{Evidence-aligned local composition}
\label{alg:main}
\begin{algorithmic}[1]
\Require masked $y$ with masked set $M$, experts $\{h_i\}_{i=1}^m$, sampler $\mathcal S$
\Require particles $P$, score samples $K$, iterations $R$, step $\eta$
\Require smoother $\mathrm{Smooth}_\tau$, decoder time $t_{\mathrm{dec}}$
\State $\lambda^{(0)}_{\ell,i}\gets 1/m$ for all $\ell,i$
\For{$s=0,\dots,R-1$}
  \State $(\mathcal X_{\mathrm{pr}},\mathcal X_{\mathrm{po}})\gets\mathcal S(h_{\lambda^{(s)}},y)$
  \State $g^{(s)}\gets$ gradient estimate \eqref{eq:est_gradient} from $\mathcal X_{\mathrm{pr}},\mathcal X_{\mathrm{po}}$
  \State $\tilde\lambda^{(s+1)}_{\ell,i}\propto\lambda^{(s)}_{\ell,i}\exp(\eta g^{(s)}_{\ell,i})$, renormalize over $i$
  \State $\lambda^{(s+1)}\gets \mathrm{Smooth}_\tau\big(\tilde\lambda^{(s+1)}\big)$
\EndFor
\State $\hat\lambda_{1:L}\gets\lambda^{(R)}_{1:L}$; initialize decoder state $z$ from $y$
\State $\hat x_\ell\gets$ decode \eqref{eq:decode} for $\ell\in M$;\ \ $\hat x_\ell\gets y_\ell$ otherwise
\State \Return $\hat\lambda_{1:L},\hat x$
\end{algorithmic}
\end{algorithm}

\section{Experiments}
\label{sec:experiments}

Our experiments address two questions: can the evidence-aligned update recover a known composition field, and does the recovered field improve restoration? We begin with an exact categorical simulator, where field recovery can be measured directly, then turn to learned byte-level experts and to experts adapted from a $1.3$B discrete flow-matching model, and close with naturally mixed documents, whose regions are not constructed for the experiment.

\paragraph{Protocol.}
Within an experiment all methods share the frozen experts, corrupted input, channel, decoder, and source initialization; only the composition weights vary. Comparisons are paired over the stated unit, aggregating by source document when a document contributes several windows, and we report $95\%$ bootstrap intervals, two-sided sign-flip permutation $p$-values, and Cohen's $d_z$. Unless noted, the channel is masking.

\paragraph{Metrics.}
We report restoration quality and field quality separately. \emph{Restoration accuracy} is the fraction of corrupted positions recovered exactly. For the field we use two measures, according to what is known. When the true field $\lambda^\star$ is available, as in the simulator of Section~\ref{sec:exp_toy}, we report the \emph{field mean absolute error} $\frac{1}{Lm}\sum_{\ell,i}|\hat\lambda_{\ell,i}-\lambda^\star_{\ell,i}|$. When only region labels are available, which is every learned-expert setting, we report \emph{field accuracy}, the fraction of positions at which $\argmax_i\hat\lambda_{\ell,i}$ equals the region label. Field accuracy is meaningful only for a position-varying field: a constant field has the same argmax everywhere, and we mark those entries $\cdot$. Region labels enter the metric only, never the method, the hard label router excepted, which is privileged for that reason.

\paragraph{Baselines.}
We compare local evidence with each individual expert, equal weights, an independently optimized global weight that enforces $\lambda_\ell=\lambda$ under the same optimization, reported at the byte, $1.3$B and arXiv scales, a best-single-expert selector given by the vertex with the lowest own-posterior denoising score, shuffled-field controls, and a hard router that receives the true region labels. The router is privileged but not an upper bound, since a soft composition can mix experts near region boundaries. Section~\ref{sec:exp_routers} adds three local routers that omit the evidence correction, and we adapt three published test-time fusion methods to corrupted windows: PackLLM in its softmax and greedy variants \citep{Mavromatis2024}, a DEMix-style per-position posterior router \citep{DEMix2022}, and a marginal-evidence selector \citep{DiME2026}. Each scores the visible tokens with the experts' own denoising losses and feeds the same decoder. Appendices~\ref{app:stats} and~\ref{app:baselines} detail every baseline and the full statistical protocol.

\subsection{Recovering a known composition field}
\label{sec:exp_toy}

We begin where the composition field is known exactly. The categorical simulator uses vocabulary size $12$, length $48$, and two experts, with three segments (two routed, one mixed); we draw a clean sequence from the composed prior and corrupt it with the replacement channel at rate $0.4$. Optimizing the exact evidence \eqref{eq:evidence} recovers the field at mean absolute error $0.013$, versus $0.333$ for equal weights and $0.500$ for a single expert (Figure~\ref{fig:field}); replacing the exact densities with two trained flow-matching experts and the denoising surrogate \eqref{eq:dscore} raises the error only to $0.031$, and the surrogate energy correlates $0.974$/$0.971$ with the exact log likelihood. Under the exact posterior decoder the recovered and ground-truth fields both reach accuracy $0.652$, above equal weights ($0.634$) and either expert ($0.382$; Table~\ref{tab:toy}). A broader sweep shows recovery improving with more observations and greater separation and degrading with corruption.

\begin{table}[t]
\centering
\caption{Field recovery and reconstruction in the categorical simulator (exact posterior decoder). Recon.: masked-token accuracy.}
\label{tab:toy}
\small
\begin{tabular}{lcc}
\toprule
Weights & MAE $\downarrow$ & Recon. $\uparrow$ \\
\midrule
Ground-truth $\lambda^\star$ & $0.000$ & $0.652$ \\
Exact evidence & $0.013$ & $0.652$ \\
Flow-matching surrogate & $0.031$ & $0.652$ \\
Equal weights & $0.333$ & $0.634$ \\
Single expert (either) & $0.500$ & $0.382$ \\
\bottomrule
\end{tabular}
\end{table}

\begin{figure}[t]
\centering
\includegraphics[width=\linewidth]{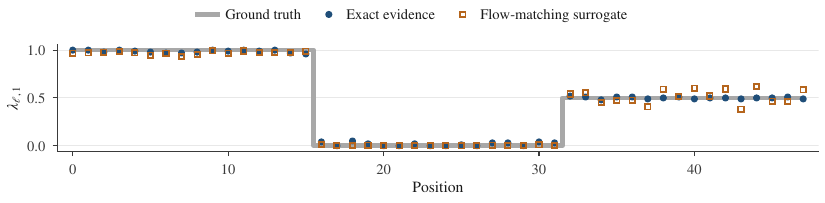}
\caption{Recovering a known field from corrupted observations alone: two routed segments (positions $0$ to $31$) and a fused segment ($32$ to $47$).}
\label{fig:field}
\end{figure}

\subsection{Byte-level experts}
\label{sec:exp_byte}

\begingroup
\setlength{\tabcolsep}{3pt}
\begin{table}[!t]
\centering
\caption{Three byte-level experts, $256$ windows from $160$ documents, mask channel, one-step decoder, equal-document estimand. Hard label router uses true labels (privileged, not an upper bound).}
\label{tab:byte3}
\begin{adjustbox}{max width=\linewidth}
\begin{tabular}{lccccc}
\toprule
Method & Accuracy $\uparrow$ & $\Delta$ vs ours & $95\%$ CI & $p$ & $d_z$ \\
\midrule
\textbf{Local evidence (ours)} & $\mathbf{0.660}$ & $\cdot$ & $\cdot$ & $\cdot$ & $\cdot$ \\
Hard label router & $0.658$ & $+0.001$ & $[-0.003,\,+0.006]$ & $0.50$ & $0.05$ \\
PackLLM-sim fusion \citep{Mavromatis2024} & $0.622$ & $+0.038$ & $[+0.032,\,+0.044]$ & $<10^{-3}$ & $0.99$ \\
PackLLM-opt fusion \citep{Mavromatis2024} & $0.619$ & $+0.041$ & $[+0.035,\,+0.046]$ & $<10^{-3}$ & $1.13$ \\
Independent global weight & $0.618$ & $+0.042$ & $[+0.037,\,+0.048]$ & $<10^{-3}$ & $1.18$ \\
Best single expert (marginal evidence) \citep{DiME2026} & $0.613$ & $+0.047$ & $[+0.041,\,+0.053]$ & $<10^{-3}$ & $1.20$ \\
Best single expert (own posterior) & $0.613$ & $+0.047$ & $[+0.041,\,+0.053]$ & $<10^{-3}$ & $1.21$ \\
DEMix router \citep{DEMix2022} & $0.612$ & $+0.047$ & $[+0.041,\,+0.054]$ & $<10^{-3}$ & $1.09$ \\
Equal weights & $0.564$ & $+0.096$ & $[+0.088,\,+0.104]$ & $<10^{-3}$ & $1.82$ \\
Shuffled field (within doc) & $0.541$ & $+0.119$ & $[+0.110,\,+0.128]$ & $<10^{-3}$ & $2.02$ \\
\bottomrule
\end{tabular}
\end{adjustbox}
\end{table}
\endgroup
Now we consider learned byte-level experts, first pairing prose and code on masked software-documentation windows. On a single corruption seed, local evidence reaches $0.534$, above the hard label router ($0.514$), best single expert ($0.507$), global weight ($0.478$), and equal weights ($0.440$); the adapted fusion baselines fall in the same band ($0.508$ PackLLM-sim, $0.504$ PackLLM-opt, $0.507$ marginal, $0.492$ DEMix). Averaged over three corruption seeds the same ordering holds within $0.003$ per method (Table~S5). The largest gain is in code regions, where accuracy rises from $0.410$ (equal) to $0.577$, likely because the soft field composes experts near a boundary rather than committing to one region. We then add a configuration expert for the main three-expert setting; each window spans at least two regimes with a minimum length so no region dominates. Across $256$ windows from $160$ documents, local evidence reaches document-weighted accuracy $0.660$ (Table~\ref{tab:byte3}), indistinguishable from the hard label router ($p=0.50$) and above the best single expert, the independently optimized global weight, equal weights, every individual expert, and the adapted literature baselines (all $p<10^{-3}$). Shuffling the field within a document costs $0.119$ and across documents $0.181$, so restoration depends on where the weights are placed, not only their average; the field recovers the three regimes at macro one-versus-rest AUROC $0.993$, and an iterative decoder preserves the ordering.

Figure~\ref{fig:flagship} shows representative windows and a failure case chosen by prespecified rules: in the two median windows the field approaches each region's expert in the interior and turns within a few positions of the boundary, while in the lowest-performing window for the harder pair it identifies the prose and code interiors but briefly weights the absent configuration expert near their transition. Errors thus concentrate at region boundaries.

\begin{figure}[t]
\centering
\includegraphics[width=\linewidth]{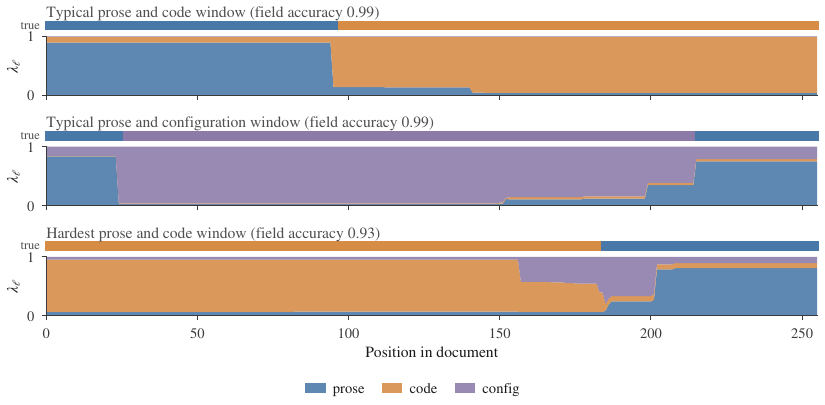}
\caption{Recovered fields for constructed windows. Ribbon: true regions; curves: weights $\lambda_\ell$ inferred from the masked window alone.}
\label{fig:flagship}
\end{figure}

\subsection{Pretrained experts at scale}
\label{sec:exp_scale}

We then ask whether the same behavior persists when the experts are adapted from a large pretrained model. We fine-tune prose and code experts with rank-$16$ LoRA from a shared $1.3$B discrete flow-matching model. The evaluation uses constructed windows sampled with replacement. The statistical unit here is the window rather than the source document, since source-document identities are not tracked in this setting; possible reuse of source material across windows is therefore a limitation.

\paragraph{These results fall outside Proposition~\ref{prop:seq_gradient}'s assumptions.} The denoising-loss argument of Section~\ref{sec:method} is stated for the mask-source path, but this $1.3$B base model has a \emph{uniform} source, so at intermediate times a corrupted position holds a random token rather than a mask symbol. We have no analogue of the identity for that path, so here the per-position denoising score is a heuristic energy rather than an order-averaged conditional log score, and the mismatch is uncontrolled. This applies to every result in this subsection and to both natural-document settings, which is to say to the results we lead with. Their high field accuracy shows the surrogate stays informative; it does not validate the identity for a uniform source. The supplement discusses the uniform-source mismatch in full.

Local evidence reaches $0.653$, above the independently optimized global weight ($0.611$; $+0.042$, $95\%$ CI $[+0.039,+0.046]$, $p<10^{-3}$, $d_z=1.47$), the best single expert ($0.599$), and equal weights ($0.599$; Table~\ref{tab:scale}); the adapted PackLLM and marginal-evidence baselines track the global family, all below local ($p<10^{-3}$), and the field agrees with the true regions at $0.995$. A single expert fine-tuned on the \emph{union} of both domains, the ``train one model on everything'' alternative, needs pooled data our setting does not assume yet reaches only $0.636$ (above the global weight, below local; $+0.017$, $p<10^{-3}$) and falls to $0.649$ on the README documents of Section~\ref{sec:natural}, below both the global weight and the best single expert. This union expert receives the same $6000$ fine-tuning steps as each specialist while covering twice the domain diversity, so it is undertrained relative to a specialist by construction and the comparison is not compute-matched; it bounds what one model trained on the pooled domains achieves at a matched step count, not at a matched budget. The independent global is not detectably different from the field-average reference ($0.611$ versus $0.610$, $p=0.72$), the privileged hard label router is $0.002$ higher than local evidence ($0.655$, $p=0.01$), and the ordering is stable across mask rates.

\begin{table}[t]
\centering
\caption{Fine-tuned $1.3$B experts, $256$ windows, mask rate $0.3$. Paired differences vs local evidence. $^{*}p<0.05$, $^{***}p<10^{-3}$.}
\label{tab:scale}
\small
\setlength{\tabcolsep}{4pt}
\begin{tabular}{@{}lcc@{}}
\toprule
Method & Acc. & $\Delta$ vs ours \\
\midrule
\textbf{Local evidence} & $\mathbf{0.653}$ & $\cdot$ \\
Hard label router & $0.655$ & $-0.002^{*}$ \\
Single generalist (union) & $0.636$ & $+0.017^{***}$ \\
PackLLM-opt & $0.616$ & $+0.037^{***}$ \\
Independent global & $0.611$ & $+0.042^{***}$ \\
PackLLM-sim & $0.605$ & $+0.048^{***}$ \\
Best single (own post.) & $0.599$ & $+0.054^{***}$ \\
Best single (marginal) & $0.599$ & $+0.055^{***}$ \\
Equal weights & $0.599$ & $+0.054^{***}$ \\
Shuffled field & $0.570$ & $+0.083^{***}$ \\
Prose expert & $0.511$ & $+0.142^{***}$ \\
\midrule
Field accuracy & \multicolumn{2}{c}{$0.995$} \\
\bottomrule
\end{tabular}
\end{table}

\subsection{Local routers without the evidence correction}
\label{sec:exp_routers}

The preceding comparisons show that position-dependent weights can help, but they do not establish that the posterior-minus-prior correction is responsible. We therefore compare against three routers that replace the evidence update with a local heuristic: normalized per-expert scores (a DEMix-style responsibility router \citep{DEMix2022}), raw denoising loss without the prior term, and a two-state HMM over the per-expert scores.

In the two pretrained settings, the responsibility, raw-loss and HMM routers perform at or below equal weights and trail local evidence ($p<10^{-3}$; Table~\ref{tab:routers}); at $1.3$B the first two land on equal weights to within $0.001$. They also trail on the arXiv documents of Section~\ref{sec:natural} ($0.653$, $0.650$ and $0.644$ against $0.691$; Table~\ref{tab:arxiv}). Their field accuracy is far lower throughout ($0.80$, $0.81$ and $0.72$ against $0.995$ at $1.3$B; $0.56$, $0.56$ and $0.51$ against $0.765$ on README). Local evidence is thus the only router that combines strong restoration with high field accuracy across settings.

We caution that this section does not isolate the posterior-minus-prior correction on its own. Each comparison router differs from local evidence in more than one way at once; the raw local-loss router, for instance, drops both the prior expectation and the imputed-particle posterior estimate, since the adapted baselines score only visible tokens. Attributing the gap to the prior term alone would need an ablation that keeps the deployed posterior term and removes only the prior expectation, which we have not run.

\begin{table}[t]
\centering
\small
\setlength{\tabcolsep}{3pt}
\caption{Local routers that omit the posterior-minus-prior correction, across the three settings; each router cell is masked-token accuracy / field accuracy. References are constant fields (accuracy only); the global weight is the independent one at byte and $1.3$B and the jointly optimized one on README, which are indistinguishable where both were run. The byte and README columns are document-weighted; the $1.3$B column takes the window as the unit, as declared in Section~\ref{sec:exp_scale}. The HMM router was not run at byte scale.}
\label{tab:routers}
\begin{tabular}{@{}lccc@{}}
\toprule
Method & Byte & $1.3$B & README \\
\midrule
\textbf{Local evidence} & $\mathbf{.66/.97}$ & $\mathbf{.65/.99}$ & $\mathbf{.71/.77}$ \\
DEMix resp. & $.61/.96$ & $.60/.80$ & $.69/.56$ \\
Raw local-loss & $.55/.57$ & $.60/.81$ & $.69/.56$ \\
HMM & n/a & $.57/.72$ & $.68/.51$ \\
\midrule
Equal (ref.) & $.56$ & $.60$ & $.71$ \\
Global weight & $.62$ & $.61$ & $.71$ \\
PackLLM-opt & $.62$ & $.62$ & $.71$ \\
\bottomrule
\end{tabular}
\end{table}

\begin{table}[t]
\centering
\small
\setlength{\tabcolsep}{4pt}
\caption{Distinct natural experts: arXiv prose and \LaTeX{} mathematics, $256$ windows from $71$ documents, mask rate $0.3$, document-clustered. Field: field accuracy ($\cdot$ = constant field). Best in each column is bold, the oracle label router excepted. $^{***}p<10^{-3}$.}
\label{tab:arxiv}
\begin{tabular}{@{}lccc@{}}
\toprule
Method & Acc. & Field & $\Delta$ \\
\midrule
Oracle labels & $0.702$ & $\cdot$ & $-0.011^{***}$ \\
\textbf{Local evidence} & $\mathbf{0.691}$ & $\mathbf{0.852}$ & $\cdot$ \\
Global weight (independent) & $0.682$ & $\cdot$ & $+0.010^{***}$ \\
PackLLM-opt & $0.679$ & $\cdot$ & $+0.012^{***}$ \\
PackLLM-sim & $0.676$ & $\cdot$ & $+0.015^{***}$ \\
Equal weights & $0.673$ & $\cdot$ & $+0.018^{***}$ \\
Best single & $0.669$ & $\cdot$ & $+0.022^{***}$ \\
Best single (marg.) & $0.666$ & $\cdot$ & $+0.024^{***}$ \\
DEMix router & $0.653$ & $0.621$ & $+0.038^{***}$ \\
Raw local-loss & $0.650$ & $0.623$ & $+0.041^{***}$ \\
HMM router & $0.644$ & $0.555$ & $+0.049^{***}$ \\
\bottomrule
\end{tabular}
\end{table}

\subsection{Naturally mixed documents}
\label{sec:natural}

Constructed mixtures have known boundaries; we now test whether local composition helps on documents whose regions arise naturally. Scientific articles interleave natural-language prose with \LaTeX{} mathematics, two lexically distant regimes that co-occur in every paper. We fine-tune prose and mathematics experts from the $1.3$B model on collected arXiv sources and label each token by its math environment; the two experts separate strongly, with margins of $0.47$ and $0.46$ nats. Across $256$ windows from $71$ documents with document-clustered statistics, local evidence reaches accuracy $0.691$ and outperforms an independently optimized global weight by $0.010$ ($p<10^{-3}$, $d_z=0.56$), equal weights by $0.018$, and the best single expert by $0.022$ (Table~\ref{tab:arxiv}), while recovering the regions at field accuracy $0.852$. The adapted global-fusion baselines match the global weight (PackLLM-opt $+0.012$, PackLLM-sim $+0.015$, marginal-evidence $+0.024$; all $p<10^{-3}$), and the local routers of Section~\ref{sec:exp_routers} trail it further. Local evidence is therefore the only method here that both improves restoration over the global family and attributes the regions faithfully.

Whether local composition improves on a global weight in natural documents appears to depend on how distinct the experts are. When they converge on the evaluation text, as with prose and code experts fine-tuned on matched README corpora, we detect no restoration gain over a global weight ($0.709$ versus $0.708$), and the field instead contributes regional attribution (field accuracy $0.765$). This is consistent with the constructed-mixture trend in which greater separation yields a larger gain, but it rests on two natural settings, so we report it as an observation rather than as a rule that predicts the regime in advance.

\section{Conclusion}
\label{sec:conclusion}

We introduced evidence-aligned local composition, a test-time method for restoring corrupted discrete sequences with frozen experts and no region supervision or trained router: instead of one expert per sequence, it infers a simplex-valued field from the corrupted observation, whose evidence gradient, for a normalized sequence energy model, is exactly a posterior-minus-prior difference of expert energies estimated from the experts' denoising losses (the deployed algorithm only approximates this, so we call it evidence-aligned). The most consistent result is recovery of position-wise expert preference: on constructed windows the field tracks the known regions and outperforms an independently optimized global weight, and on natural documents the gain tracks expert separation, improving over the global family on distinct prose and \LaTeX{} mathematics and reducing to attribution on converged README prose and code. Three limitations. First, the schedule-weighted likelihood identity requires a Bayes-optimal mask-source denoiser on supported contexts, whereas the deployed unweighted score remains a surrogate, the composed-logit sampler is not known to draw from $Q_\lambda$, and the $1.3$B experts have a uniform source for which we have no analogue of the identity at all. Second, composition helps when the experts are specialized and distinct, its advantage over a global weight vanishing when they converge. Third, our comparisons do not isolate the posterior-minus-prior correction on its own, since every local router we compare against differs from local evidence in more than one way at once; an ablation that keeps the deployed posterior term and removes only the prior expectation remains open.

\section*{Acknowledgments}
Shuchin Aeron would like to acknowledge support by NSF under Cooperative Agreement PHY-2019786 and NSF DMS 2309519. 

\newpage

\bibliography{references}

@inproceedings{Austin2021,
  author    = {Austin, J. and Johnson, D. D. and Ho, J. and Tarlow, D. and van den Berg, R.},
  title     = {Structured denoising diffusion models in discrete state-spaces},
  booktitle = {NeurIPS},
  year      = {2021}
}

@inproceedings{DiME2026,
  author    = {Wang, F. and Bouman, K. L.},
  title     = {Sample-efficient evidence estimation of score-based priors for model selection},
  booktitle = {International Conference on Learning Representations (ICLR)},
  year      = {2026},
  note      = {arXiv:2602.20549}
}

@misc{FactorDiff2026,
  author        = {Huang, H. and Xu, Y. and Mandal, A. and Aspuru-Guzik, A.},
  title         = {From global to factor-wise expert composition in discrete diffusion models},
  year          = {2026},
  eprint        = {2607.11758},
  archivePrefix = {arXiv}
}

@misc{WangBouman2026,
  author        = {Wang, F. and Bouman, K. L.},
  title         = {Optimizing diffusion priors in image reconstruction from a single observation},
  year          = {2026},
  eprint        = {2604.21066},
  archivePrefix = {arXiv}
}

@inproceedings{wPoE2025,
  author    = {Zhang, Q. and Li, M. and Wang, Z. and Liao, R. and Wang, L.},
  title     = {Test-time steering for lossless text compression via weighted product of experts},
  booktitle = {Findings of the Association for Computational Linguistics: EMNLP},
  pages     = {2076--2088},
  year      = {2025},
  note      = {arXiv:2511.10660}
}

@inproceedings{Campbell2022,
  author    = {Campbell, A. and Benton, J. and De Bortoli, V. and Rainforth, T. and Deligiannidis, G. and Doucet, A.},
  title     = {A continuous time framework for discrete denoising models},
  booktitle = {NeurIPS},
  year      = {2022}
}

@inproceedings{DEMix2022,
  author    = {Gururangan, S. and Lewis, M. and Holtzman, A. and Smith, N. A. and Zettlemoyer, L.},
  title     = {{DEMix} layers: Disentangling domains for modular language modeling},
  booktitle = {NAACL},
  year      = {2022}
}

@inproceedings{Du2023,
  author    = {Du, Y. and Durkan, C. and Strudel, R. and Tenenbaum, J. B. and Dieleman, S. and Fergus, R. and Sohl-Dickstein, J. and Doucet, A. and Grathwohl, W.},
  title     = {Reduce, reuse, recycle: Compositional generation with energy-based diffusion models and {MCMC}},
  booktitle = {ICML},
  year      = {2023}
}

@inproceedings{Gat2024,
  author    = {Gat, I. and Remez, T. and Shaul, N. and Kreuk, F. and Chen, R. T. Q. and Synnaeve, G. and Adi, Y. and Lipman, Y.},
  title     = {Discrete flow matching},
  booktitle = {NeurIPS},
  year      = {2024}
}

@article{Hinton2002,
  author  = {Hinton, G. E.},
  title   = {Training products of experts by minimizing contrastive divergence},
  journal = {Neural Computation},
  volume  = {14},
  number  = {8},
  year    = {2002}
}

@inproceedings{Hoogeboom2022,
  author    = {Hoogeboom, E. and Gritsenko, A. A. and Bastings, J. and Poole, B. and van den Berg, R. and Salimans, T.},
  title     = {Autoregressive diffusion models},
  booktitle = {ICLR},
  year      = {2022}
}

@inproceedings{Liu2022,
  author    = {Liu, N. and Li, S. and Du, Y. and Torralba, A. and Tenenbaum, J. B.},
  title     = {Compositional visual generation with composable diffusion models},
  booktitle = {ECCV},
  year      = {2022}
}

@inproceedings{Lou2024,
  author    = {Lou, A. and Meng, C. and Ermon, S.},
  title     = {Discrete diffusion modeling by estimating the ratios of the data distribution},
  booktitle = {ICML},
  year      = {2024}
}

@article{MacKay1992,
  author  = {MacKay, D. J. C.},
  title   = {Bayesian interpolation},
  journal = {Neural Computation},
  volume  = {4},
  number  = {3},
  year    = {1992}
}

@misc{Mavromatis2024,
  author        = {Mavromatis, C. and Karypis, P. and Karypis, G.},
  title         = {Pack of {LLMs}: Model fusion at test-time via perplexity optimization},
  year          = {2024},
  eprint        = {2404.11531},
  archivePrefix = {arXiv}
}

@inproceedings{Ou2024,
  author    = {Ou, J. and Nie, S. and Xue, K. and Zhu, F. and Sun, J. and Li, Z. and Li, C.},
  title     = {Your absorbing discrete diffusion secretly models the conditional distributions of clean data},
  booktitle = {ICLR},
  year      = {2025},
  note      = {arXiv:2406.03736}
}

@inproceedings{Sahoo2024,
  author    = {Sahoo, S. S. and Arriola, M. and Schiff, Y. and Gokaslan, A. and Marroquin, E. and Chiu, J. T. and Rush, A. and Kuleshov, V.},
  title     = {Simple and effective masked diffusion language models},
  booktitle = {NeurIPS},
  year      = {2024}
}

@inproceedings{Shi2024,
  author    = {Shi, J. and Han, K. and Wang, Z. and Doucet, A. and Titsias, M. K.},
  title     = {Simplified and generalized masked diffusion for discrete data},
  booktitle = {NeurIPS},
  year      = {2024}
}

@inproceedings{Uria2014,
  author    = {Uria, B. and Murray, I. and Larochelle, H.},
  title     = {A deep and tractable density estimator},
  booktitle = {ICML},
  year      = {2014}
}

@inproceedings{Lipman2023,
  author    = {Lipman, Y. and Chen, R. T. Q. and Ben-Hamu, H. and Nickel, M. and Le, M.},
  title     = {Flow matching for generative modeling},
  booktitle = {ICLR},
  year      = {2023}
}

@misc{Nie2025LLaDA,
  author        = {Nie, S. and Zhu, F. and You, Z. and Zhang, X. and Ou, J. and Hu, J. and Zhou, J. and Lin, Y. and Wen, J.-R. and Li, C.},
  title         = {Large language diffusion models},
  year          = {2025},
  eprint        = {2502.09992},
  archivePrefix = {arXiv}
}

@inproceedings{Gong2025,
  author    = {Gong, S. and Agarwal, S. and Zhang, Y. and Ye, J. and Zheng, L. and Li, M. and An, C. and Zhao, P. and Bi, W. and Peng, H. and Han, J. and Kong, L.},
  title     = {Scaling diffusion language models via adaptation from autoregressive models},
  booktitle = {ICLR},
  year      = {2025}
}

@inproceedings{Nie2025Scaling,
  author    = {Nie, S. and Zhu, F. and Du, C. and Pang, T. and Liu, Q. and Zeng, G. and Lin, M. and Li, C.},
  title     = {Scaling up masked diffusion models on text},
  booktitle = {ICLR},
  year      = {2025}
}

@misc{Dream2025,
  author        = {Ye, J. and Xie, Z. and Zheng, L. and Gao, J. and Wu, Z. and Jiang, X. and Li, Z. and Kong, L.},
  title         = {Dream 7{B}: Diffusion large language models},
  year          = {2025},
  eprint        = {2508.15487},
  archivePrefix = {arXiv}
}

@inproceedings{Nisonoff2025,
  author    = {Nisonoff, H. and Xiong, J. and Allenspach, S. and Listgarten, J.},
  title     = {Unlocking guidance for discrete state-space diffusion and flow models},
  booktitle = {ICLR},
  year      = {2025}
}

@inproceedings{Schiff2025,
  author    = {Schiff, Y. and Sahoo, S. S. and Phung, H. and Wang, G. and Boshar, S. and Dalla-Torre, H. and Almeida, B. P. and Rush, A. M. and Pierrot, T. and Kuleshov, V.},
  title     = {Simple guidance mechanisms for discrete diffusion models},
  booktitle = {ICLR},
  year      = {2025}
}

@inproceedings{Garipov2023,
  author    = {Garipov, T. and De Peuter, S. and Yang, G. and Garg, V. and Kaski, S. and Jaakkola, T.},
  title     = {Compositional sculpting of iterative generative processes},
  booktitle = {NeurIPS},
  year      = {2023}
}

@inproceedings{Skreta2025,
  author    = {Skreta, M. and Atanackovic, L. and Bose, A. J. and Tong, A. and Neklyudov, K.},
  title     = {The superposition of diffusion models using the {It\^o} density estimator},
  booktitle = {ICLR},
  year      = {2025}
}

@inproceedings{Chung2023,
  author    = {Chung, H. and Kim, J. and McCann, M. T. and Klasky, M. L. and Ye, J. C.},
  title     = {Diffusion posterior sampling for general noisy inverse problems},
  booktitle = {ICLR},
  year      = {2023}
}

@inproceedings{Kawar2022,
  author    = {Kawar, B. and Elad, M. and Ermon, S. and Song, J.},
  title     = {Denoising diffusion restoration models},
  booktitle = {NeurIPS},
  year      = {2022}
}

@inproceedings{Wan2024,
  author    = {Wan, F. and Huang, X. and Cai, D. and Quan, X. and Bi, W. and Shi, S.},
  title     = {Knowledge fusion of large language models},
  booktitle = {ICLR},
  year      = {2024}
}

@inproceedings{Jiang2023,
  author    = {Jiang, D. and Ren, X. and Lin, B. Y.},
  title     = {{LLM-Blender}: Ensembling large language models with pairwise ranking and generative fusion},
  booktitle = {ACL},
  year      = {2023}
}

@misc{Sukhbaatar2024,
  author        = {Sukhbaatar, S. and Golovneva, O. and Sharma, V. and Xu, H. and Lin, X. V. and Rozi\`ere, B. and Kahn, J. and Li, D. and Yih, W.-t. and Weston, J. and Li, X.},
  title         = {Branch-train-{MiX}: Mixing expert {LLMs} into a mixture-of-experts {LLM}},
  year          = {2024},
  eprint        = {2403.07816},
  archivePrefix = {arXiv}
}

@inproceedings{Merity2017,
  title={Pointer Sentinel Mixture Models},
  author={Merity, Stephen and Xiong, Caiming and Bradbury, James and Socher, Richard},
  booktitle={International Conference on Learning Representations},
  year={2017}
}

@article{Husain2019,
  title={CodeSearchNet Challenge: Evaluating the State of Semantic Code Search},
  author={Husain, Hamel and Wu, Ho-Hsiang and Gazit, Tiferet and Allamanis, Miltiadis and Brockschmidt, Marc},
  journal={arXiv preprint arXiv:1909.09436},
  year={2019}
}

@misc{TheStackYamlK8s,
  title={the-stack-yaml-k8s: Kubernetes YAML manifests from The Stack},
  author={{Substratus AI}},
  year={2023},
  howpublished={Hugging Face dataset, \url{https://huggingface.co/datasets/substratusai/the-stack-yaml-k8s}}
}

@misc{CodeParrotClean,
  title={codeparrot-clean: a deduplicated Python subset of GitHub code},
  author={{CodeParrot}},
  year={2022},
  howpublished={Hugging Face dataset, \url{https://huggingface.co/datasets/codeparrot/codeparrot-clean}}
}

@misc{ProofPile,
  title={proof-pile: a corpus of mathematical text},
  author={{Hoskinson Center for Formal Mathematics}},
  year={2022},
  howpublished={Hugging Face dataset, \url{https://huggingface.co/datasets/hoskinson-center/proof-pile}}
}

@misc{GithubMarkdown,
  title={github-top1000-repos-markdown: README and Markdown files from the 1000 most-starred GitHub repositories},
  author={{Riabov, Maksim}},
  year={2024},
  howpublished={Hugging Face dataset, \url{https://huggingface.co/datasets/MRiabov/github-top1000-repos-markdown}}
}

@article{Tibshirani2005,
  title={Sparsity and smoothness via the fused lasso},
  author={Tibshirani, Robert and Saunders, Michael and Rosset, Saharon and Zhu, Ji and Knight, Keith},
  journal={Journal of the Royal Statistical Society: Series B},
  volume={67},
  number={1},
  pages={91--108},
  year={2005}
}

@article{Condat2013,
  title={A direct algorithm for 1-{D} total variation denoising},
  author={Condat, Laurent},
  journal={IEEE Signal Processing Letters},
  volume={20},
  number={11},
  pages={1054--1057},
  year={2013}
}
\bibliographystyle{unsrtnat}

\newpage
\begin{appendices}

\renewcommand{\theequation}{S\arabic{equation}}\setcounter{equation}{0}
\renewcommand{\thetable}{S\arabic{table}}\setcounter{table}{0}
\renewcommand{\thefigure}{S\arabic{figure}}\setcounter{figure}{0}
\makeatletter
\def\theHequation{S\arabic{equation}}
\def\theHtable{S\arabic{table}}
\def\theHfigure{S\arabic{figure}}
\makeatother

Equations, tables and figures introduced in this appendix carry an \textsf{S}
prefix and appendix sections are lettered; plain numbers refer to the main
text. Proposition numbering continues from the main text, whose
Proposition~\mpPropSeqGradient{} is proved below.

\bigskip

\section{Proofs}
\label{app:proofs}

Appendix~\ref{app:gradient} proves Proposition~\mpPropSeqGradient{} of the main text, the sequence-level evidence gradient that the method rests on. Appendix~\ref{app:factorized} proves the position-factorized special case, which is the form used in the controlled simulator.

\subsection{The sequence-level evidence gradient}
\label{app:gradient}

We first fix notation and state the standing assumptions, then prove the identity in full.

\paragraph{Setting.}
The vocabulary $\V$ is finite and the length $L$ is fixed, so the state space $\V^L$ is finite. Each energy $\phi_{i,\ell}:\V^L\to\R$ is finite valued. For $\lambda\in\R^{L\times m}$ write
\begin{equation}
\begin{gathered}
\Phi_\lambda(x)=\textstyle\sum_{\ell=1}^{L}\sum_{i=1}^{m}\lambda_{\ell,i}\,\phi_{i,\ell}(x),\\
Z(\lambda)=\textstyle\sum_{x'\in\V^L}\exp\Phi_\lambda(x'),
\end{gathered}
\label{eq:app_seqpool}
\end{equation}
so that the composed prior of \mpEqSeqPool{} is $Q_\lambda(x)=\exp\Phi_\lambda(x)/Z(\lambda)$. The channel $\pc(y\mid x)$ is arbitrary subject to $p_\lambda(y)=\sum_x\pc(y\mid x)Q_\lambda(x)>0$, and the posterior is $Q_\lambda(x\mid y)=\pc(y\mid x)Q_\lambda(x)/p_\lambda(y)$. No structure on the channel is needed: it need not be memoryless, and $\phi_{i,\ell}$ may depend on the whole sequence.

\begin{proof}[Proof of Proposition~\mpPropSeqGradient]
\emph{Step 1: the composed prior is positive and smooth in $\lambda$.}
Each $\Phi_\lambda(x')$ is linear in $\lambda$ with finite coefficients $\phi_{i,\ell}(x')$, so $\exp\Phi_\lambda(x')$ is positive and infinitely differentiable in $\lambda$ on all of $\R^{L\times m}$. As $\V^L$ is finite, $Z(\lambda)$ is a finite sum of such terms, hence $Z(\lambda)>0$ and $Z$ is smooth. Therefore $Q_\lambda(x)>0$ and $\lambda\mapsto Q_\lambda(x)$ is smooth for every $x$.

\emph{Step 2: the evidence is smooth, so the derivative exists.}
$p_\lambda(y)=\sum_{x}\pc(y\mid x)Q_\lambda(x)$ is a finite sum of smooth functions of $\lambda$ and is therefore smooth. It is positive by assumption, so $\log p_\lambda(y)$ is smooth, in particular differentiable on the relative interior of $\simplex_m^L$, and the posterior $Q_\lambda(\cdot\mid y)$ is a well defined distribution on $\V^L$. Every interchange of differentiation and summation below is over a finite index set and so needs no further justification.

\emph{Step 3: the score of the composed prior.}
Since $\partial\Phi_\lambda(x)/\partial\lambda_{\ell,i}=\phi_{i,\ell}(x)$, differentiating $Z$ term by term gives
\[\begin{aligned}
\frac{\partial\log Z(\lambda)}{\partial\lambda_{\ell,i}}
&=\frac{1}{Z(\lambda)}\textstyle\sum_{x'}\exp\big(\Phi_\lambda(x')\big)\,\phi_{i,\ell}(x')\\
&=\E_{x'\sim Q_\lambda}\big[\phi_{i,\ell}(x')\big].
\end{aligned}\]
Because $\log Q_\lambda(x)=\Phi_\lambda(x)-\log Z(\lambda)$, subtracting yields
\begin{equation}
\frac{\partial\log Q_\lambda(x)}{\partial\lambda_{\ell,i}}=\phi_{i,\ell}(x)-\E_{x'\sim Q_\lambda}\big[\phi_{i,\ell}(x')\big].
\label{eq:prior_score}
\end{equation}
This is the exponential-family identity: the derivative of the log normalizer is the mean of the sufficient statistic, here the energy $\phi_{i,\ell}$.

\emph{Step 4: differentiate the evidence.}
Differentiating $p_\lambda(y)$ term by term and using $\partial Q_\lambda=Q_\lambda\,\partial\log Q_\lambda$, which is legitimate because $Q_\lambda(x)>0$ by Step 1,
\[\begin{aligned}
\frac{\partial\log p_\lambda(y)}{\partial\lambda_{\ell,i}}
&=\frac{1}{p_\lambda(y)}\textstyle\sum_{x}\pc(y\mid x)\,\frac{\partial Q_\lambda(x)}{\partial\lambda_{\ell,i}}\\
&=\textstyle\sum_{x}\underbrace{\frac{\pc(y\mid x)\,Q_\lambda(x)}{p_\lambda(y)}}_{=\;Q_\lambda(x\mid y)}\\
&\qquad\times\,\frac{\partial\log Q_\lambda(x)}{\partial\lambda_{\ell,i}}\\
&=\E_{x\sim Q_\lambda(\cdot\mid y)}\Big[\frac{\partial\log Q_\lambda(x)}{\partial\lambda_{\ell,i}}\Big].
\end{aligned}\]
The bracketed weight is nonnegative and sums to one over $x$, so the last line is an expectation under the posterior.

\emph{Step 5: conclude.}
Substitute \eqref{eq:prior_score}. The prior expectation $\E_{x'\sim Q_\lambda}[\phi_{i,\ell}(x')]$ does not depend on $x$, so it passes through the posterior expectation unchanged, leaving
\[\begin{aligned}
\frac{\partial\log p_\lambda(y)}{\partial\lambda_{\ell,i}}
&=\E_{x\sim Q_\lambda(\cdot\mid y)}\big[\phi_{i,\ell}(x)\big]\\
&\quad-\E_{x\sim Q_\lambda}\big[\phi_{i,\ell}(x)\big],
\end{aligned}\]
which is \mpEqSeqGradient{} and proves the proposition.
\end{proof}

\subsection{The factorized special case}
\label{app:factorized}

Throughout this subsection we fix a position $\ell$ and suppress it when no ambiguity arises. Expert $i$ supplies the categorical distribution $\pi_{i,\ell}$ at position $\ell$, and $a_i(v)=\log\pi_{i,\ell}(v)$. For $\lambda_\ell\in\simplex_m$, the composed local prior is the normalized log-linear pool from \mpEqPool:
\begin{equation}
\begin{gathered}
\pi_{\lambda_\ell,\ell}(v)=\frac{\exp A_{\lambda_\ell}(v)}{Z_{\lambda_\ell}},\\
A_{\lambda_\ell}(v)=\textstyle\sum_{j=1}^m \lambda_{\ell,j}\,a_j(v),\quad Z_{\lambda_\ell}=\textstyle\sum_{u\in\V}\exp A_{\lambda_\ell}(u).
\end{gathered}
\label{eq:app_pool}
\end{equation}
We assume every expert assigns positive probability to every token, so that $a_i(v)$ is finite for all $i$ and $v$. This is the only regularity condition required for the local derivative arguments below, and it holds for the softmax parameterization of the experts we use.

We call the corruption channel \emph{memoryless} when it acts on each position independently, so that it factorizes as $\pc(y\mid x)=\prod_{\ell=1}^{L}\pc(y_\ell\mid x_\ell)$. Both channels of Section~\mpSecSetup{} are memoryless: the mask channel replaces each token by $\mathtt{[MASK]}$ independently with probability $r$, and the replacement channel replaces each token by an independent uniform draw from $\V$.

For position-factorized experts the composed prior factorizes in the same way, $Q_\lambda(x)=\prod_{\ell}\pi_{\lambda_\ell,\ell}(x_\ell)$. Substituting both factorizations into the marginal likelihood and exchanging the sum over sequences for a product of per-position sums,
\begin{equation*}
\begin{aligned}
p_\lambda(y)&=\sum_{x\in\V^L}\prod_{\ell=1}^{L}\pc(y_\ell\mid x_\ell)\,\pi_{\lambda_\ell,\ell}(x_\ell)\\
&=\prod_{\ell=1}^{L}\sum_{v\in\V}\pc(y_\ell\mid v)\,\pi_{\lambda_\ell,\ell}(v),
\end{aligned}
\end{equation*}
where the second equality is the distributive law: a sum over all sequences of a product of per-position factors equals the product of the per-position sums. Taking logarithms, $\log p_\lambda(y)=\sum_{\ell}\log\sum_{v\in\V}\pc(y_\ell\mid v)\,\pi_{\lambda_\ell,\ell}(v)$ is a sum of per-position terms, the $\ell$-th of which depends on $\lambda$ only through $\lambda_\ell$. Averaging over $N$ shared observations $y^{(1)},\dots,y^{(N)}$ and writing $c_{n,\ell}(v)=\pc(y^{(n)}_\ell\mid x_\ell=v)$ recovers the per-position objective $L_\ell(\lambda_\ell)$ of \mpEqEvidence, whose maximization decouples across positions.

Its gradient can therefore be read at a single position. Fix $\ell$ and, for observation $n$, define the channel posterior over the clean token as
\begin{equation}\gamma_{n,\ell}(v)=\frac{c_{n,\ell}(v)\,\pi_{\lambda_\ell,\ell}(v)}{\sum_{u} c_{n,\ell}(u)\,\pi_{\lambda_\ell,\ell}(u)}. \label{eq:token_post}\end{equation}

\begin{proposition}[Local evidence gradient]
\label{prop:gradient}
For $\lambda_\ell$ in the relative interior of $\simplex_m$, the objective \mpEqEvidence{} is differentiable and
\begin{equation}
\begin{aligned}
\frac{\partial L_\ell(\lambda_\ell)}{\partial \lambda_{\ell,i}}={}&\tfrac1N\textstyle\sum_{n=1}^N \E_{v\sim\gamma_{n,\ell}}\big[a_i(v)\big]\\
&-\E_{v\sim\pi_{\lambda_\ell,\ell}}\big[a_i(v)\big].
\end{aligned}
\label{eq:gradient}
\end{equation}
\end{proposition}

We first isolate the derivative of the composed prior, which is where the normalizer of \eqref{eq:app_pool} enters and which is the source of the prior term in the final expression.

\begin{lemma}[Score of the log-linear pool]
\label{lem:pool_score}
For every $\lambda_\ell$ in the relative interior of $\simplex_m$ and every $v\in\V$,
\begin{equation}
\begin{aligned}
&\frac{\partial}{\partial\lambda_{\ell,i}}\log\pi_{\lambda_\ell,\ell}(v)\\
&\qquad=a_i(v)-\E_{u\sim\pi_{\lambda_\ell,\ell}}\big[a_i(u)\big].
\end{aligned}
\label{eq:pool_score}
\end{equation}
\end{lemma}

\begin{proof}
By \eqref{eq:app_pool}, $\log\pi_{\lambda_\ell,\ell}(v)=A_{\lambda_\ell}(v)-\log Z_{\lambda_\ell}$. The first term is linear in $\lambda_\ell$ with $\partial A_{\lambda_\ell}(v)/\partial\lambda_{\ell,i}=a_i(v)$. For the second, the sum defining $Z_{\lambda_\ell}$ is finite and each summand is a positive, smooth function of $\lambda_\ell$, so we may differentiate term by term:
\[\begin{aligned}
\frac{\partial}{\partial\lambda_{\ell,i}}\log Z_{\lambda_\ell}
&=\frac{1}{Z_{\lambda_\ell}}\textstyle\sum_{u\in\V}\exp\big(A_{\lambda_\ell}(u)\big)a_i(u)\\
&=\textstyle\sum_{u\in\V}\pi_{\lambda_\ell,\ell}(u)\,a_i(u)=\E_{u\sim\pi_{\lambda_\ell,\ell}}\big[a_i(u)\big].
\end{aligned}\]
Subtracting gives \eqref{eq:pool_score}.
\end{proof}

Equation~\eqref{eq:pool_score} is the standard exponential-family identity: the derivative of the log-normalizer is the mean sufficient statistic under the current pool. The score is invariant to expert-specific additive constants in $a_i$; this invariance does not remove scale differences or state-dependent surrogate errors. It also has zero mean under the current pool, $\E_{v\sim\pi_{\lambda_\ell,\ell}}[\partial_{\lambda_{\ell,i}}\log\pi_{\lambda_\ell,\ell}(v)]=0$, which is the standard score-function identity.

\begin{proof}[Proof of Proposition~\ref{prop:gradient}]
Fix an observation index $n$ and write
\[E_{n,\ell}(\lambda_\ell)=\textstyle\sum_{v\in\V} c_{n,\ell}(v)\,\pi_{\lambda_\ell,\ell}(v),\quad \mathcal J_{n,\ell}=\log E_{n,\ell},\]
so that $L_\ell(\lambda_\ell)=N^{-1}\sum_{n=1}^N \mathcal J_{n,\ell}(\lambda_\ell)$ by \mpEqEvidence. Since $\V$ is finite and every $\pi_{\lambda_\ell,\ell}(v)$ is positive and smooth in the relative interior of the simplex, $E_{n,\ell}$ is a finite sum of smooth functions. Although $c_{n,\ell}$ need not have a positive uniform lower bound, at least one token satisfies $c_{n,\ell}(v)>0$. Full support gives $\pi_{\lambda_\ell,\ell}(v)>0$ for that token, so $E_{n,\ell}>0$ and $\log E_{n,\ell}$ is differentiable.

Using $\partial \pi = \pi\,\partial\log\pi$,
\[\begin{aligned}
\frac{\partial \mathcal J_{n,\ell}}{\partial\lambda_{\ell,i}}
&=\frac{1}{E_{n,\ell}(\lambda_\ell)}\textstyle\sum_{v\in\V} c_{n,\ell}(v)\,\partial_{\lambda_{\ell,i}} \pi_{\lambda_\ell,\ell}(v)\\
&=\textstyle\sum_{v\in\V}\underbrace{\frac{c_{n,\ell}(v)\,\pi_{\lambda_\ell,\ell}(v)}{E_{n,\ell}(\lambda_\ell)}}_{=\;\gamma_{n,\ell}(v)}\\
&\qquad\times\,\partial_{\lambda_{\ell,i}} \log\pi_{\lambda_\ell,\ell}(v),
\end{aligned}\]
where $\gamma_{n,\ell}$ is precisely the channel posterior \eqref{eq:token_post}. The bracketed factor is a probability distribution on $\V$ by construction, so the sum is an expectation under $\gamma_{n,\ell}$. Substituting Lemma~\ref{lem:pool_score},
\[\begin{aligned}
\frac{\partial \mathcal J_{n,\ell}}{\partial\lambda_{\ell,i}}
&=\E_{v\sim\gamma_{n,\ell}}\Big[a_i(v)-\E_{u\sim\pi_{\lambda_\ell,\ell}}[a_i(u)]\Big]\\
&=\E_{v\sim\gamma_{n,\ell}}\big[a_i(v)\big]-\E_{u\sim\pi_{\lambda_\ell,\ell}}\big[a_i(u)\big],
\end{aligned}\]
the last step because the inner expectation does not depend on $v$. Averaging over $n=1,\dots,N$ gives \eqref{eq:gradient}.
\end{proof}

\begin{remark}[Reading the two terms]
The gradient compares the same statistic $a_i$ under the channel posterior $\gamma_{n,\ell}$ and the current composed prior $\pi_{\lambda_\ell,\ell}$. An expert gains relative weight when it scores observation-compatible tokens above tokens from the current composition. This centering removes expert-specific additive offsets but does not remove relative scale or general calibration differences.
\end{remark}

\begin{proposition}[Degeneracy of the factorized objective under masking]
\label{prop:mask_degeneracy}
Suppose position $\ell$ is masked in observation $n$, so that the channel likelihood is fixed, $c_{n,\ell}(v)=r$ for all $v\in\V$. Then observation $n$ contributes nothing to the gradient \eqref{eq:gradient} at position $\ell$.
\end{proposition}

\begin{proof}
If $c_{n,\ell}\equiv r$ then $E_{n,\ell}(\lambda_\ell)=r\sum_v \pi_{\lambda_\ell,\ell}(v)=r$, which is constant in $\lambda_\ell$, and the channel posterior reduces to the prior, $\gamma_{n,\ell}(v)=r\,\pi_{\lambda_\ell,\ell}(v)/r=\pi_{\lambda_\ell,\ell}(v)$. The two expectations in \eqref{eq:gradient} are then taken under the same distribution and cancel.
\end{proof}

This formalizes the masking difficulty noted in Section~\mpSecMethod: in the factorized model, a masked observation contributes no local information about its position-specific weight. Another observation in which the position is visible may contribute information. For contextual experts, $\lambda_\ell$ can affect the marginal distribution of visible tokens through cross-position dependence, so the sequence-level evidence can remain informative. The practical estimator also scores imputed posterior particles and smooths weights across positions.

\subsection{Contextual $2\times2$ validation}
\label{app:ctx2x2}

To examine the energy and sampling approximations separately, we use a contextual model small enough to enumerate. It is a normalized bigram-energy prior with two experts whose local energies depend on both $x_{\ell-1}$ and $x_\ell$; unlike a factorized model, it therefore exposes a genuine difference between the surrogate and exact energies. The short sequence and small vocabulary allow us to compute $Q_\lambda$, its posterior, and its gradient exactly, as well as a mean-field law that samples positions independently from the composed local conditionals. Table~\ref{tab:ctx2x2} crosses exact or surrogate energies with exact or mean-field expectations.

The exact sequence-level gradient agrees with a finite-difference gradient at cosine $1.000$, a numerical implementation check for Proposition~\mpPropSeqGradient{} (Table~\ref{tab:ctx2x2}). Replacing the exact energies with the denoising surrogate, or the exact expectations with mean-field sampling, perturbs the gradient direction, the energy surrogate more (cosine $0.65$) than the sampler ($0.92$); the deployed combination of both reaches cosine $0.62$. This enumerable model isolates the two approximations but does not by itself establish tolerance to them at scale. The reported spreads are standard deviations across the $20$ seeds; the two surrogate-energy configurations vary widely, so their means should be read as averages over seeds whose alignment ranges from weak to near exact.

\paragraph{Field recovery in the contextual model.}
The comparison above is taken at the uniform initialization and measures agreement of the gradient, not whether the objective identifies the field. We therefore run the recovery update to convergence in the same enumerable model, with the exact gradient, varying the number of corrupted observations $N$ that share one field. The true field routes the first half of the positions to one expert and the second half to the other, softened off the vertices, so a uniform field sits at mean absolute error $0.455$. Recovery improves monotonically with $N$ and closes about three quarters of that distance (Table~\ref{tab:ctxrecover}). A single observation of length four carries too little information to separate the four configurations of Table~\ref{tab:ctx2x2}, which is why we report gradient agreement rather than field error there. This is the contextual counterpart of the factorized recovery result of Section~\mpSecToy{} and of its dependence on the number of observations.

\begin{table}[t]
\centering
\small
\caption{Field recovery in the enumerable contextual model}
\label{tab:ctxrecover}
\begin{tabular}{@{}lcccccc@{}}
\toprule
Observations $N$ & $1$ & $4$ & $16$ & $64$ & $256$ & $1024$ \\
\midrule
Field MAE $\downarrow$ & $.362$ & $.316$ & $.181$ & $.136$ & $.120$ & $.099$ \\
\bottomrule
\end{tabular}
\end{table}

\begingroup
\setlength{\tabcolsep}{3pt}
\begin{table}[t]
\centering
\small
\caption{Enumerable diagnostic on a contextual bigram-energy model (4 tokens, length 4, 256 sequences, 20 seeds): an implementation check, not evidence that the large-scale approximation is controlled. The mean-field law differs from the exact prior and posterior at total variation $0.38$ and $0.34$.}
\label{tab:ctx2x2}
\begin{adjustbox}{max width=\linewidth}
\begin{tabular}{lcc}
\toprule
Configuration & Gradient cosine $\uparrow$ & Gradient RMSE $\downarrow$ \\
\midrule
(1) exact energy, exact sampler & $1.000$ & $0.000$ \\
(2) surrogate energy, exact sampler & $0.646\pm0.359$ & $0.259\pm0.161$ \\
(3) exact energy, approx.\ sampler & $0.923\pm0.090$ & $0.129\pm0.071$ \\
(4) surrogate energy, approx.\ sampler (deployed) & $0.618\pm0.352$ & $0.304\pm0.170$ \\
\bottomrule
\end{tabular}
\end{adjustbox}
\end{table}
\endgroup

\section{Experimental details}
\label{app:details}

This section contains the data construction, expert configurations, optimization settings, decoders, and statistical procedures used in the main experiments. The implementation, the evaluation-set construction scripts and the analysis code are at \url{https://github.com/panahazari/evidence-aligned-local-composition}. No corpus is redistributed there; the evaluation sets are rebuilt from the public sources listed below.

\subsection{Data and corruption}
\label{app:data}

\paragraph{Exact simulator (Section~\mpSecToy).}
Vocabulary $|\V|=12$, length $L=48$, $m=2$ experts. Each expert is a position-factorized categorical model with per-position logits drawn so that the two experts are separated by a logit-gap parameter, whose default value $2.5$ corresponds to a measured Kullback-Leibler divergence of $4.56$ nats between the experts. The separation parameter is not itself a divergence; figures report the measured KL. The ground-truth field routes positions $0$ to $15$ to the first expert, positions $16$ to $31$ to the second, and fuses both on positions $32$ to $47$. Clean sequences are drawn from the composed prior and corrupted by the replacement channel at rate $0.4$. We use $3000$ corrupted sequences for recovery.

\paragraph{Byte-level experts (Section~\mpSecByte).}
We use three corpora, each truncated to $5$M characters: prose from WikiText-103 (raw) \citep{Merity2017}, Python code from CodeSearchNet \citep{Husain2019} with comments and docstrings removed by a tokenizer-based filter, and Kubernetes YAML \citep{TheStackYamlK8s} with full-line comments removed. These filters reduce trivial overlap between domains. Evaluation windows concatenate segments from the evaluation split of each corpus without artificial fences or boundary markers, so region labels are known by construction. The three-expert study uses $256$ windows of $L=256$ bytes, each containing at least two regimes with a minimum length for each. The two-expert study uses $128$ prose and code windows.

\paragraph{Pretrained experts (Section~\mpSecScale).}
For the $1.3$B study, evaluation windows of $L=512$ tokens are built the same way from WikiText-103 prose \citep{Merity2017} and CodeParrot Python code \citep{CodeParrotClean} held apart from the fine-tuning data, giving $256$ windows. For the $7$B study we use $64$ windows in the shared tokenizer of the two models.

\paragraph{Real documents (Section~\mpSecNatural).}
The natural evaluation contains $256$ windows from $75$ developer-written README documents drawn from the most-starred GitHub repositories \citep{GithubMarkdown}. We tokenize each document as a whole and retain its original prose and code regions and boundaries. Because a source document may contribute several windows, every statistic in this study resamples documents rather than individual windows.

\paragraph{Scientific documents (Section~\mpSecNatural).}
The distinct-expert natural setting uses arXiv \LaTeX{} sources (the proof-pile corpus, \citealp{ProofPile}). We split each document into a prose stream and a mathematics stream by its math environments (inline \$\dots\$, display \$\$\dots\$\$ and \texttt{\textbackslash[\dots\textbackslash]}, and \texttt{equation}/\texttt{align}-type environments) and fine-tune one expert on each stream, holding documents out for evaluation. Each held-out document is tokenized whole with the GPT-2 tokenizer, and each token is labeled prose or mathematics by whether its characters fall inside a math environment; a token straddling a boundary takes the majority label and is excluded from region-conditional metrics. We keep windows in which both regions clear a $0.25$ share, giving $256$ windows from $71$ documents at mask rate $0.3$, with a $0.50$ mathematics share and a token-level region-switch rate about seven times higher than the README documents.

\paragraph{Channels.}
The mask channel replaces each token independently with probability $r$ by $\mathtt{[MASK]}$. The default is $r=0.3$ for the pretrained studies and $r=0.2$ for the byte studies, and Appendix~\ref{app:rate} sweeps it. The replacement channel replaces each token independently with probability $r$ by a uniform draw from the vocabulary; we use it in the exact simulator, where the posterior decoder is exact.

\subsection{Experts}
\label{app:experts}

\paragraph{Byte-level experts.}
Each expert is a discrete flow-matching model over the byte vocabulary, $|\V|=256$ plus a mask symbol, so $257$ classes. All experts in a composition here share an identical architecture, which is convenient but not required (see the expert interface in Section~\mpSecSetup); what the logit composition needs is a shared vocabulary and output interpretation, not a shared network. The architecture is a transformer with $8$ layers, model width $384$, $12$ attention heads, dropout $0.1$, and context $L=256$, about $14$M parameters. Training uses the mask-source path with a polynomial schedule ($\kappa$ exponent $1.0$, $t_\epsilon=10^{-4}$), batch size $48$, $20{,}000$ steps, AdamW with learning rate $3\times10^{-4}$, weight decay $0.01$, and $400$ warmup steps. Each expert sees only its own domain. Final training cross-entropies are $2.40$ (prose), $1.97$ (code) and $1.83$ (configuration) nats per byte.

\paragraph{Fine-tuned $1.3$B experts.}
Both experts are fine-tuned from the same $1.3$B discrete flow-matching model with a GPT-2 vocabulary ($50{,}257$ tokens). We use LoRA with rank $16$ and $\alpha=32$, learning rate $10^{-4}$, batch size $16$, $6000$ steps, mixed precision, and gradient checkpointing on $12{,}000$ blocks of $512$ tokens per domain. Sharing the base model and fine-tuning recipe reduces architectural confounding and makes the domain data the intended source of specialization. Section~\mpSecNatural{} also uses a pair fine-tuned on matched corpora and the unmodified base model as a generalist.

\paragraph{arXiv prose and mathematics experts.}
The scientific-document experts of Section~\mpSecNatural{} are fine-tuned from the same $1.3$B model with the identical LoRA recipe ($6000$ steps), one on the extracted prose stream and one on the mathematics stream. On the arXiv evaluation regions they pass the specialization check with balanced margins of $0.47$ (prose) and $0.46$ (mathematics) nats, the strongest two-sided separation among the natural-document pairs. The README-matched pair, by contrast, separates only weakly on its evaluation text, which is why local composition beats a global weight in the scientific setting and ties in the README setting.

\paragraph{Off-the-shelf $7$B experts.}
We use \texttt{Dream-org/Dream-v0-Base-7B} as the general expert and \texttt{Dream-org/Dream-Coder-v0-Base-7B} as the code expert at mask rate $0.6$. These public masked-diffusion language models share a tokenizer and compatible source and output conventions. Strictly speaking, they are trained as masked diffusion models rather than under the discrete flow-matching parameterization used elsewhere in the paper. The distinction does not affect the method. Both frameworks expose the same two operations required by the expert interface of Section~\mpSecSetup: clean-token logits given a partially masked context and a denoising loss on a candidate sequence. Moreover, the denoising-loss identity of Section~\mpSecMethod{} is stated for the mask-source path and its conclusion does not depend on the schedule, so it covers mask-source masked diffusion and mask-source discrete flow matching alike (Appendix~\ref{app:source}). We therefore refer to all experts as discrete flow-matching models in the main text and note the training-objective difference only here. We do not train or adapt them, and each runs on a separate GPU. The pair fails the specialization diagnostic in Appendix~\ref{app:specialization}. The composed fill is deterministic, so the eight nominal particles are identical and the estimator is a plug-in rather than a particle average.

\subsection{The specialization check}
\label{app:specialization}

The inferred field depends on per-expert scale, so we use a specialization diagnostic before composition. We evaluate every expert's denoising energy on every domain and accept the diagnostic when each expert assigns its lowest energy to its own domain. For the three byte-level experts of Section~\mpSecByte, the energies in nats per byte are

{\centering
\begin{tabular}{lccc}
\toprule
& \multicolumn{3}{c}{Energy under expert} \\
\cmidrule(lr){2-4}
Data domain & Prose & Code & Configuration \\
\midrule
Prose         & $\mathbf{1.27}$ & $2.46$ & $2.30$ \\
Code          & $1.55$ & $\mathbf{1.02}$ & $1.42$ \\
Configuration & $1.55$ & $1.32$ & $\mathbf{0.81}$ \\
\bottomrule
\end{tabular}\par}

\noindent with diagonal margins of $1.03$, $0.40$ and $0.51$ nats. The byte-level experts and both $1.3$B pairs pass this check.

\paragraph{The $7$B pair fails the diagnostic.} For the two off-the-shelf models, the measured energies are $0.888$ nats per token for Dream on prose and $0.514$ on code, compared with $1.078$ for Dream-Coder on prose and $0.337$ on code. Dream-Coder assigns lower energy to code by $0.74$ nats, but Dream also assigns lower energy to code by $0.37$ nats. The matrix is therefore not diagonally dominant and the diagnostic fails.

Column-wise domain discrimination still holds: Dream has lower energy than Dream-Coder on prose, while Dream-Coder has lower energy than Dream on code. Row-wise own-domain preference fails because Dream assigns lower energy to code than to prose. The field nevertheless reaches field accuracy $0.95$. We therefore report the experiment as a stress test outside the diagnostic's accepted regime.

Appendix~\ref{app:seeds} shows the complementary case: a byte-level pair that fails to specialize also fails to produce a usable field, which is what the check is for.

\subsection{Composition hyperparameters}
\label{app:hparams}

Hyperparameters differ across experiments. Table~\ref{tab:settings} reports the settings used in each run. The byte-level studies use $R=30$, $48$ particles, $K=4$, $\eta=1$, and $\tau=0.6$. The $1.3$B studies use $R=12$, $24$ particles, $K=2$, $\eta=0.5$, and $\tau=0.4$ under the proximal map, or $\tau=0.3$ under the moving average on naturally mixed documents; the $7$B study uses $R=6$, eight nominal particles, and $\tau=0.2$. The field is initialized uniformly on the simplex and the seed is $0$.

Only two settings were searched. The smoothing strength $\tau$ was swept over $\{0,0.2,0.6\}$ at byte scale, over $\{0,0.1,0.2,0.4,0.6,1,2\}$ at $1.3$B and $\{0,0.2,0.6,1\}$ at $7$B under the proximal map (Appendix~\ref{app:tv}), and the corruption rate over $\{0.1,0.2,0.3\}$ at byte scale and $\{0.3,0.5,0.7\}$ at $1.3$B (Appendix~\ref{app:rate}); the rate grid is reported as an ablation rather than used to pick a winner. The deployed $\tau$ at byte scale was chosen once from the byte sweep, which was run on the same windows we report. Read literally, that means the byte $\tau$ was selected on the reported evaluation data, with no calibration subset kept apart, and the byte numbers should be read with that in mind. The $1.3$B and $7$B protocol below does not have this problem. At $1.3$B and $7$B it was chosen by a rule fixed before the sweep was run, namely the highest restoration accuracy on a small calibration set held apart from the reported windows, with field accuracy as the tie-break. At $1.3$B this selects $\tau=0.4$ from a plateau that is flat to $0.0003$ between $0.2$ and $0.6$. At $7$B the rule as written selects $\tau=0$, by $0.0007$ over $16$ windows, but that setting leaves field accuracy at $0.745$ against $0.967$ at $\tau=0.2$; we treat the restoration difference as a tie and deploy $\tau=0.2$, which is a deviation from the rule and is the only such deviation in the paper. The remaining settings ($R$, $P$, $K$, $\eta$, sampler steps) were fixed in advance by the compute budget and were not tuned: one value was used per scale throughout, as listed in Table~\ref{tab:settings}. No held-out selection was performed for any hyperparameter.

The smoothing operator is experiment-specific. The reference total-variation operator is the simplex-constrained proximal map
\begin{equation}
  \begin{aligned}
    \lambda^{(s+1)}=\argmin_{\lambda_{1:L}\in\simplex_m^L}\Big\{&\tfrac12\textstyle\sum_{\ell}\|\lambda_\ell-\tilde\lambda^{(s+1)}_\ell\|_2^2\\
    &+\tau \textstyle\sum_{\ell\ge2}\|\lambda_\ell-\lambda_{\ell-1}\|_1\Big\} .
  \end{aligned}
  \label{eq:tv}
\end{equation}
For the two-expert byte-level study, we apply the exact one-dimensional fused-lasso map to the first weight, clip it to $[0,1]$, and set the second weight to its complement; this implements \eqref{eq:tv}. For the three-expert study, we apply the same map independently to each weight column, clip negative values, and renormalize each row to the simplex. The constructed $1.3$B and $7$B studies use the same two-expert fused-lasso map as the byte-level study, so \eqref{eq:tv} is the operator behind every constructed result in the paper. For the studies on naturally mixed documents, we instead smooth each weight column with a moving average, using a centered rectangular kernel of width five whose five taps are all fixed to $1/5$, with zero padding at the two ends. The new field is $(1-\tau)$ times the current field plus $\tau$ times its moving average, and each row is then renormalized to the simplex. Thus, $\tau$ is a total-variation penalty in the constructed studies and the weight given to the moving average in the naturally mixed document studies. The two operators do not share one proximal objective, and we make no common convergence claim; within each study the same operator is applied to the evidence field and to every baseline field it is compared against.

\begingroup
\setlength{\tabcolsep}{3pt}
\begin{table}[t]
\centering
\small
\setlength{\tabcolsep}{4pt}
\caption{Settings used by each experiment; they are \emph{not} shared across scales. $R$: recovery iterations; $P$: particles; $K$: time samples; $\eta$: step size; $\tau$: smoothing strength.}
\label{tab:settings}
\begin{adjustbox}{max width=\linewidth}
\begin{tabular}{lccccccccc}
\toprule
Experiment & Rate & $L$ & $R$ & $P$ & $K$ & $\eta$ & $\tau$ & Windows & Docs \\
\midrule
Byte, 2 experts & $0.2$ & $256$ & $30$ & $48$ & $4$ & $1.0$ & $0.6$ & $128$ & 64 \\
Byte, 3 experts & $0.2$ & $256$ & $30$ & $48$ & $4$ & $1.0$ & $0.6$ & $256$ & 160 \\
Fine-tuned 1.3B & $0.3$ & $512$ & $12$ & $24$ & $2$ & $0.5$ & $0.4$ & $256$ & n/a \\
Off-the-shelf 7B & $0.6$ & $512$ & $6$ & $8^{\dagger}$ & $2$ & $0.5$ & $0.2$ & $64$ & 64 \\
Natural, README & $0.3$ & $512$ & $12$ & $24$ & $2$ & $0.5$ & $0.3$ & $256$ & 75 \\
Natural, arXiv & $0.3$ & $512$ & $12$ & $24$ & $2$ & $0.5$ & $0.3$ & $256$ & 71 \\
\bottomrule
\end{tabular}
\end{adjustbox}
\par\vspace{2pt}
{\footnotesize\raggedright
$^{\dagger}$The $7$B composed fill is deterministic, so its $8$ particles are identical; that estimator is a plug-in, not a particle average. The first four rows use the proximal map of \eqref{eq:tv}, where $\tau$ is a penalty; the two natural-document rows use the moving-average blend, where $\tau$ is a mixing weight. The two are not comparable.\par}
\end{table}
\endgroup

\subsection{Decoders}
\label{app:decoders}

The default decoder is the one-step reconstruction \mpEqDecode{} at $t_{\mathrm{dec}}=0.9$: visible tokens are copied and each masked token is filled independently by the argmax of the composed logits. The iterative decoder revisits positions over several denoising times, using $12$ steps in the byte studies and $6$ rounds at $1.3$B. All methods use the same decoder and matched source initialization. This controls decoding variation but does not make the decoder cancel, because decoding is nonlinear in the composition field. Appendix~\ref{app:decoder} shows that the method ordering is unchanged.

\subsection{Protocol and statistics}
\label{app:stats}

Every method uses the same experts, corrupted input, channel, decoder, and matched source initialization; only the composition weights differ. Comparisons are paired over the declared statistical unit. Byte-level and natural-document analyses aggregate and resample by source document. The $1.3$B analysis uses constructed windows and takes the window as the statistical unit, as noted in Section~\mpSecScale. We report the mean paired difference from local evidence, its $95\%$ bootstrap confidence interval, a two-sided sign-flip permutation $p$-value, and Cohen's $d_z$, each with $10{,}000$ resamples and a fixed seed. Restoration accuracy and field accuracy are as defined in Section~\ref{sec:experiments}.

\subsection{Compute}
\label{app:compute}

All experiments run on NVIDIA A100 40GB GPUs. Training one byte-level expert takes under an hour on a single GPU. Fine-tuning one $1.3$B expert takes a few hours on a single GPU. Composition is the dominant cost at inference: with the settings above, the three-expert byte-level run takes about $65$ seconds per document, dominated by the particle draws and the per-expert energy evaluations, and the full $256$-window run takes about $4.6$ GPU-hours. The $7$B study places one expert per GPU. Costs scale linearly in the number of experts, since each expert is scored independently at each iteration.

\section{Additional results}
\label{app:extra}

The following experiments test sensitivity to corruption, expert training, decoder choice, and domain granularity, and provide more detailed measurements of the recovered field.

\subsection{Corruption rate and seeds}
\label{app:rate}

Table~\ref{tab:rate} reports a corruption-rate sweep. The byte-level study averages three corruption seeds; the spread is below $0.007$ accuracy for every method. Its method ordering is unchanged across rates, while the advantage over equal weights narrows from $+0.101$ at $r=0.1$ to $+0.078$ at $r=0.3$. The $1.3$B sweep also preserves the ordering through $r=0.7$, with a margin near $0.05$ over equal weights. However, the $1.3$B rates use different evaluation subsets and sample sizes, so this sweep provides descriptive evidence of stable ordering rather than a controlled estimate of the effect of corruption rate.

\begin{figure}[t]
\centering
\includegraphics[width=\linewidth]{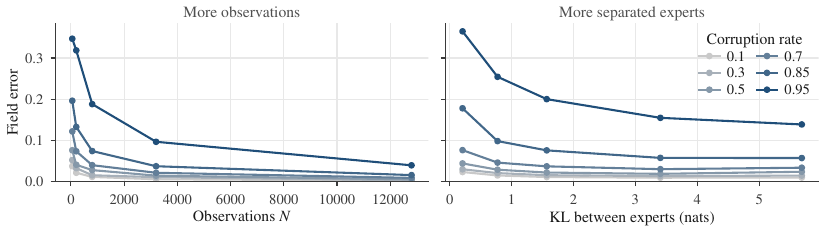}
\caption{Recoverability of the field in the categorical simulator: error falls with more observations and greater expert separation, and rises with corruption.}
\label{fig:phase}
\end{figure}

\begin{table}[t]
\centering
\caption{Corruption-rate sweep. Byte numbers average three corruption seeds ($128$ windows each, two experts); $1.3$B uses $256$ windows at $r=0.3$ and $128$ at the higher rates.}
\label{tab:rate}
\begin{adjustbox}{max width=\linewidth}
\begin{tabular}{llccccc}
\toprule
Experts & Rate & Local (ours) & Hard label router & Best single & Global & Equal \\
\midrule
Byte-level  & $r=0.1$ & $\mathbf{0.586}$ & $0.563$ & $0.553$ & $0.523$ & $0.485$ \\
            & $r=0.2$ & $\mathbf{0.533}$ & $0.515$ & $0.505$ & $0.476$ & $0.442$ \\
            & $r=0.3$ & $\mathbf{0.479}$ & $0.463$ & $0.456$ & $0.430$ & $0.401$ \\
\midrule
$1.3$B      & $r=0.3$ & $0.653$ & $\mathbf{0.655}$ & $0.599$ & $0.612$ & $0.599$ \\
            & $r=0.5$ & $0.527$ & $\mathbf{0.534}$ & $0.475$ & $0.489$ & $0.471$ \\
            & $r=0.7$ & $0.369$ & $\mathbf{0.381}$ & $0.332$ & $0.344$ & $0.322$ \\
\bottomrule
\end{tabular}
\end{adjustbox}
\end{table}

\subsection{Independently trained expert pairs}
\label{app:seeds}

We retrained the byte-level expert pair twice with different seeds and repeated the full evaluation (Table~\ref{tab:seeds}). The first pair reproduces the main result. For the second, field accuracy is $0.51$, near chance for two experts, and local evidence ($0.449$) is close to the best single expert ($0.447$). This failed composition coincides with a failed specialization diagnostic. The result supports treating expert specialization as an empirical prerequisite, although it does not isolate specialization as the sole cause. Two retrainings are also too few to characterize seed stability, and one of the two failing is better read as a warning than as a rate.

The second pair also puts the hard label router ($0.326$) well below equal weights ($0.410$). This is what the diagnostic predicts rather than an anomaly: the router commits each position to the expert trained on that region's domain, so it is only privileged when the label tracks competence. When the pair does not specialize, the label no longer identifies the better expert at that position, and hard routing discards the averaging that equal weights still enjoys. The same mechanism makes the router an imperfect reference in general, which is why we describe it as privileged but not an upper bound.

\begin{table}[t]
\centering
\caption{Two independently retrained byte-level expert pairs. The second pair does not specialize and loses the advantage over the best single expert. \emph{Read each row against its own baselines, not across rows.} Both pairs use the same $128$ documents, decoder and protocol, but they are different models with different absolute skill: Pair 2's code expert alone reaches $0.447$ where Pair 1's reaches $0.339$, so Pair 2's higher local-evidence number reflects stronger experts, not better composition. The within-pair margins are the comparable quantity.}
\label{tab:seeds}
\small
\setlength{\tabcolsep}{3.5pt}
\begin{adjustbox}{max width=\linewidth}
\begin{tabular}{@{}lcccccc@{}}
\toprule
Expert pair & Field acc. & Local (ours) & Hard label router & Best single & Global & Equal \\
\midrule
Pair 1 (specialized)  & $0.92$ & $\mathbf{0.362}$ & $0.361$ & $0.339$ & $0.306$ & $0.290$ \\
Pair 2 (not specialized) & $0.51$ & $0.449$ & $0.326$ & $0.447$ & $0.446$ & $0.410$ \\
\bottomrule
\end{tabular}
\end{adjustbox}
\end{table}

\subsection{A second real domain}
\label{app:python}

We also evaluate a finer-grained domain pair: Python code versus its comments and docstrings. These regimes interleave at the line level within a file rather than appearing as long blocks. Across $128$ windows, local evidence reaches $0.227$. We detect no difference from the hard label router ($0.227$), while local evidence exceeds the code expert ($0.222$), the global weight ($0.212$), and equal weights ($0.188$). Absolute accuracy and margins are lower than in the main study, consistent with the greater vocabulary overlap between code and comments.

\subsection{Decoder choice}
\label{app:decoder}

Table~\ref{tab:decoder} repeats the two main studies with an iterative decoder that revisits positions across several denoising times. The saved iterative outputs support only window-weighted aggregation, so this is a secondary descriptive analysis rather than the primary document-weighted byte analysis. The method ordering is unchanged, although individual accuracies can increase or decrease; at $1.3$B the iterative decoder brings local evidence level with the privileged label router, which the one-step decoder leaves $0.002$ ahead. All methods use the same decoder and matched source initialization, which controls decoding variation but does not remove decoder effects because decoding is nonlinear in the field.

\begin{table}[t]
\centering
\caption{One-step versus iterative decoding (window-weighted). The method ordering is identical under both decoders; the best single expert is the own-posterior selector (the code expert at $1.3$B).}
\label{tab:decoder}
\begin{adjustbox}{max width=\linewidth}
\begin{tabular}{llcccc}
\toprule
Experts & Decoder & Local (ours) & Hard label router & Best single & Equal \\
\midrule
Byte-level, 3 experts & one-step  & $\mathbf{0.658}$ & $0.657$ & $0.612$ & $0.563$ \\
                      & iterative & $\mathbf{0.662}$ & $0.660$ & $0.618$ & $0.556$ \\
\midrule
$1.3$B, 2 experts     & one-step  & $0.653$ & $\mathbf{0.655}$ & $0.599$ & $0.599$ \\
                      & iterative & $\mathbf{0.656}$ & $\mathbf{0.656}$ & $0.602$ & $0.602$ \\
\bottomrule
\end{tabular}
\end{adjustbox}
\end{table}

\subsection{The smoothing prior}
\label{app:tv}

We isolate the role of the smoothing prior (Table~\ref{tab:tv}, Figure~\ref{fig:tv}). Its effect differs markedly across the two ablations. At byte scale, removing smoothing lowers local-evidence accuracy to $0.348$, below the best single expert at $0.413$; setting $\tau=0.2$ raises accuracy to $0.430$, above the reported single-expert and global baselines. On the $32$-window calibration subset of the $1.3$B study, by contrast, the unsmoothed field already reaches $0.659$, compared with $0.636$ for the field-average global reference. Raising $\tau$ to $0.4$ lifts restoration accuracy to $0.679$ and field accuracy from $0.912$ to $0.993$. The response is a plateau rather than a peak: accuracy varies by $0.0003$ across $\tau\in\{0.2,0.4,0.6\}$ and decays slowly outside it, to $0.676$ at $\tau=2$, so the deployed value is not a sharp choice. This contrast with the byte scale is consistent with reduced reliance on smoothing when the expert signal is stronger, but two ablations are not sufficient to establish a general scaling law.

\begin{table}[t]
\centering
\caption{Effect of smoothing strength $\tau$. Byte numbers use this ablation's decoder (not comparable with Table~\mpTabByte); the $1.3$B numbers use one $32$-window calibration subset, disjoint from the reported windows. Both scales use the proximal update \eqref{eq:tv}, so $\tau$ means the same thing in every column. Note that the deployed byte setting is $\tau=0.6$, which this sweep does not cover: the byte columns bracket it at $0$ and $0.2$ but do not contain it, so the deployed configuration does not appear in its own ablation.}
\label{tab:tv}
\small
\setlength{\tabcolsep}{4pt}
\begin{adjustbox}{max width=\linewidth}
\begin{tabular}{@{}lcc@{\hspace{8pt}}ccc@{}}
\toprule
& \multicolumn{2}{c}{Byte-level experts} & \multicolumn{3}{c}{Fine-tuned $1.3$B experts} \\
\cmidrule(lr){2-3}\cmidrule(lr){4-6}
& $\tau=0$ & $\tau=0.2$ & $\tau=0$ & $\tau=0.4$ & $\tau=2$ \\
\midrule
Local evidence (accuracy)   & $0.348$ & $\mathbf{0.430}$ & $0.659$ & $\mathbf{0.679}$ & $0.676$ \\
Best single expert          & $0.413$ & $0.413$          & $0.623$ & $0.623$ & $0.623$ \\
Global weight               & $0.337$ & $0.357$          & $0.636$ & $0.636$ & $0.635$ \\
Local evidence (field acc.) & $\cdot$ & $\cdot$        & $0.912$ & $0.993$ & $0.993$ \\
\bottomrule
\end{tabular}
\end{adjustbox}
\end{table}

\begin{figure}[t]
\centering
\includegraphics[width=\linewidth]{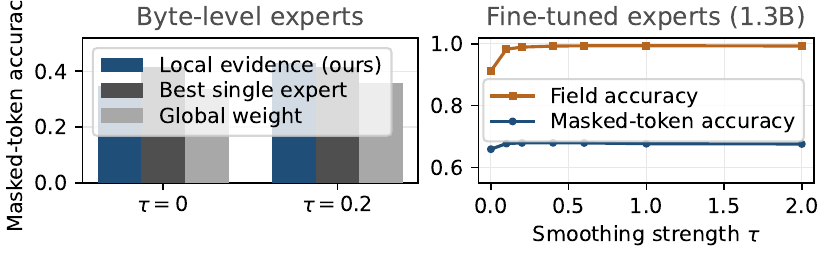}
\caption{Effect of smoothing. Without it, byte-level local evidence falls below the best single expert; at $1.3$B smoothing mainly improves field accuracy (Table~\ref{tab:tv}).}
\label{fig:tv}
\end{figure}

\subsection{Quality of the recovered field}
\label{app:field}

Restoration accuracy does not measure the quality of the inferred composition, and overall field accuracy can hide boundary errors because region interiors contain many more positions than transitions. We therefore evaluate the three-expert byte-level field against the true regions using several metrics, with clustering by source document as in the main byte-level analysis (Table~\ref{tab:fieldmetrics}). Macro one-versus-rest AUROC is $0.993$ (prose $0.985$, code $0.996$, configuration $0.996$), and class-wise argmax recall is $0.924$ for prose, $0.995$ for code, and $0.998$ for configuration.

The field is highly accurate away from transitions: interior accuracy is $0.987$, macro F1 is $0.972$, IoU is $0.948$, and segment F1 is $0.979$. Performance is lower near transitions. Accuracy within four positions of a true change is $0.687$, and one-to-one boundary F1 at the same tolerance is $0.720$. Thus, the field identifies region interiors more reliably than exact transition locations, consistent with the qualitative examples in Figure~\mpFigFlagship. This suggests boundary-aware smoothing or decoding as a useful direction.

The global and equal-weight references are constant fields and therefore contain no predicted boundaries. Their boundary F1 is zero, and their field accuracies ($0.669$ and $0.334$) reflect the prevalence of the region favored by each constant field. These rows provide context but cannot test boundary recovery.

\begin{table}[t]
\centering
\caption{Attribution metrics for the recovered field, three byte-level experts, $256$ windows from $160$ documents. Boundaries are within four positions of a regime change; segment F1 uses same-label IoU $\ge 0.5$. The constant references predict no boundaries.}
\label{tab:fieldmetrics}
\small
\setlength{\tabcolsep}{3pt}
\begin{adjustbox}{max width=\linewidth}
\begin{tabular}{@{}lccccccc@{}}
\toprule
Field & Acc. & Masked & Boundary & Interior & Macro F1 & Segment F1 & Boundary F1 \\
\midrule
\textbf{Local evidence (ours)} & $\mathbf{0.976}$ & $0.972$ & $0.687$ & $0.987$ & $0.972$ & $0.979$ & $0.720$ \\
Field-average global           & $0.669$ & $0.662$ & $0.502$ & $0.675$ & $0.399$ & $0.645$ & $0.000$ \\
Equal weights                  & $0.334$ & $0.340$ & $0.498$ & $0.327$ & $0.246$ & $0.008$ & $0.000$ \\
\bottomrule
\end{tabular}
\end{adjustbox}
\end{table}

\subsection{Off-the-shelf $7$B stress test}
\label{app:sevenb}

We include one experiment with models that we did not train, treating it as an out-of-assumption stress test rather than evidence for the method. The pair consists of Dream-v0-Base-7B as the general expert and Dream-Coder-v0-Base-7B as the code expert, evaluated at mask rate $0.6$ (Appendix~\ref{app:experts}). Three properties limit the study. First, the pair fails our specialization diagnostic: the general expert scores code $0.37$ nats better than prose (Appendix~\ref{app:specialization}), so using the uncalibrated choice $\alpha_i=1$ is not well supported. Second, the composed fill is deterministic, making the eight nominal particles identical and reducing the estimator to a single plug-in evaluation. Third, the study contains only $64$ windows.

Under these conditions, local evidence reaches $0.605$ and is not detectably different from the hard label router ($0.606$, $p=0.90$). It exceeds the general model alone by $0.024$, the best single expert by $0.011$, equal weights by $0.006$, and the field-average global reference by $0.006$, each at $p<0.01$ (Table~\ref{tab:sevenb}). The adapted baselines fit the same picture: the DEMix router is indistinguishable from local evidence ($0.605$, $p=0.83$), while PackLLM-opt ($0.597$), PackLLM-sim ($0.593$), and the marginal-evidence selector ($0.592$) trail slightly, consistent with a pair whose specialization is too weak for any weighting scheme to separate strongly. The margin over a shuffled field is the thinnest in the paper ($+0.004$, $p=0.04$), so on this pair the positional arrangement of the field carries little of the gain. The field agrees with the regions at $0.954$. These numbers apply only to this model pair and do not establish robustness to arbitrary off-the-shelf experts; they show that, for two closely related but imperfectly specialized models, the inferred field produces a small gain without reducing restoration accuracy.

\begin{table}[t]
\centering
\caption{Off-the-shelf $7$B stress test (Dream-v0-Base-7B, Dream-Coder-v0-Base-7B), $64$ windows, mask rate $0.6$. The pair fails the specialization check, so this is a stress test, not evidence for the method.}
\label{tab:sevenb}
\small
\begin{adjustbox}{max width=\linewidth}
\begin{tabular}{@{}lcc@{}}
\toprule
Method & Acc. & $\Delta$ vs ours \\
\midrule
\textbf{Local evidence}   & $\mathbf{0.605}$ & $\cdot$ \\
Hard label router         & $0.606$ & $-0.000$ \\
DEMix router \citep{DEMix2022} & $0.605$ & $+0.001$ \\
Shuffled field            & $0.602$ & $+0.004^{*}$ \\
Equal weights             & $0.600$ & $+0.006^{**}$ \\
Field-avg.\ global        & $0.599$ & $+0.006^{***}$ \\
PackLLM-opt fusion \citep{Mavromatis2024} & $0.597$ & $+0.009^{***}$ \\
Best single expert        & $0.594$ & $+0.011^{***}$ \\
PackLLM-sim fusion \citep{Mavromatis2024} & $0.593$ & $+0.012^{***}$ \\
Best single (marginal) \citep{DiME2026} & $0.592$ & $+0.013^{***}$ \\
Dream (general)           & $0.581$ & $+0.024^{***}$ \\
\midrule
Field accuracy            & \multicolumn{2}{c}{$0.954$} \\
\bottomrule
\end{tabular}
\end{adjustbox}
\par\smallskip {\footnotesize\raggedright $^{*}p<0.05$, $^{**}p<0.01$, $^{***}p<10^{-3}$; unmarked differences are not significant.\par}
\end{table}

\subsection{A setting where the composition does not matter}
\label{app:replace}

We also evaluate the byte-level study under the replacement channel. At corruption rates from $0.1$ to $0.3$, every method, including a single expert, achieves accuracy within $0.015$ of $1-r$. The one-step decoder usually copies the observed token because it provides strong evidence about the clean token, so changing the composition weights rarely changes the mode. This setting therefore does not test the field. In contrast, the exact simulator in Section~\mpSecToy{} uses an exact posterior decoder under the same channel, allowing the prior composition to affect the output. Evaluating replacement corruption at byte scale requires a decoder that can overturn the observed token.

\subsection{The uniform-source mismatch}
\label{app:source}

The denoising-loss identity of Section~\mpSecMethod{} is stated for the mask-source path: the intermediate state is the clean sequence with a random subset of positions replaced by $\mathtt{[MASK]}$. The byte-level experts and the $7$B models are of this kind. The $1.3$B experts are not: they are fine-tuned from a base model whose source distribution is \emph{uniform} over the vocabulary, so at intermediate times a corrupted position holds a random token rather than a mask symbol.

We do not have an analogue of that identity for the uniform-source path. The per-position denoising score is therefore a heuristic energy rather than an order-averaged conditional log score, and the resulting source mismatch is uncontrolled. The $1.3$B and natural-document results are empirical evaluations outside the proposition's assumptions. Their high field accuracy shows that the surrogate remains informative in these experiments, but it does not validate the mask-source identity for a uniform source. Deriving a uniform-source analogue or fine-tuning the experts with a mask-source path would address this gap.

\subsection{Independent baseline implementations}
\label{app:baselines}

We distinguish independently optimized baselines from references derived from the local field.

\paragraph{Global weight.}
The independent global baseline solves the same approximate optimization under the constraint $\lambda_\ell=\lambda$ for all $\ell$. It initializes uniformly and draws its own prior and posterior particles under its current global field at every iteration. At byte scale it reaches document-weighted accuracy $0.618$, compared with $0.660$ for local evidence, a paired difference of $+0.042$ ($p<10^{-3}$, $d_z=1.18$). At $1.3$B it reaches $0.611$, compared with $0.653$ for local evidence ($+0.042$, $95\%$ CI $[+0.039,+0.046]$, $p<10^{-3}$, $d_z=1.47$, over the $256$ shared windows), and is not detectably different from the field-average reference at $0.610$ (difference $+0.0004$, $p=0.72$). On the arXiv documents it reaches document-weighted accuracy $0.682$ over $256$ windows from $71$ documents, against $0.691$ for local evidence ($+0.010$, $p<10^{-3}$, $d_z=0.56$), and is not detectably different from the jointly optimized weight ($0.682$; difference $-0.0006$, $95\%$ CI $[-0.0029,+0.0017]$, $p=0.61$). The independently selected best single expert reaches $0.669$ there, matching the reported selector. For the $7$B and README studies we report the jointly optimized reference in place of this independent baseline, as their tables indicate.

\paragraph{Best single expert.}
At byte scale, each vertex draws its own posterior particles and is scored with its own denoising energy; the selected expert reaches $0.613$. The $1.3$B independent selector uses the same own-posterior procedure and reaches $0.599$. These are posterior-score heuristics rather than marginal-evidence estimates because they omit the free-energy and posterior-entropy terms. For the $7$B and natural-document studies we use the posterior-score reference based on particles from the local field, as their tables indicate.

\paragraph{Adapted literature baselines.}
Three further baselines adapt published test-time fusion methods to the corrupted-window setting. All three score the corrupted window's \emph{visible} tokens under each expert with the same denoising surrogate used everywhere else in this paper; corrupted positions carry no target and are excluded. Every resulting weight field feeds the same decoder as the other methods.

\emph{PackLLM} \citep{Mavromatis2024} chooses one global weight per window from the input's own likelihood. The sim variant sets $\lambda_i\propto\exp(-\bar{\mathcal L}_i/T)$, where $\bar{\mathcal L}_i$ is expert $i$'s mean visible-token score and $T=0.1$ nats. The opt variant orders experts by individual fit, then mixes in each next expert through a line search over mixing coefficients that minimizes the composed model's visible-token cross-entropy; candidate mixtures are evaluated under the same weighted-logit composition the decoder uses, on the same sampled noisings, so the optimized objective matches the deployed operator. At the $7$B scale the greedy variant optimizes the log-probability mixture used by that decoder.

\emph{DEMix routing} \citep{DEMix2022} computes a parameter-free posterior over experts from the data itself. Our adaptation sets per-position weights proportional to $\exp(-\beta\,\mathcal L_{i,\ell})$ with $\beta=1$ on the visible-token scores; a corrupted position, which has no score of its own, inherits the mean score of the nearest visible positions within an expanding window. The field then receives the same smoothing as every other local field.

\emph{HMM routing} is a third local router, reported in Section~\mpSecRouters{} and in the arXiv table. It fits a two-state hidden Markov model over the per-position per-expert visible-token scores, with a symmetric transition prior, and takes the posterior state marginal as the field. It is the segmentation-aware member of the router family, and it trails local evidence at every scale we ran it ($0.568$ at $1.3$B, $0.644$ on arXiv, $0.684$ on README). It was not run at byte scale.

\emph{Marginal-evidence selection} picks the expert with the lowest mean visible-token score. Under the mask channel, the evidence of the corrupted window factorizes into a weight-independent channel term and the marginal likelihood of the visible tokens, so this selector estimates the single-expert marginal evidence, in the spirit of evidence-based model selection for score-based priors \citep{DiME2026}. It differs from the own-posterior selector above, which scores each expert on its own reconstructions; the two selectors are reported side by side.

\end{appendices}
\end{document}